\documentclass{article}
\pdfoutput=1
\PassOptionsToPackage{numbers, compress}{natbib}

\usepackage{natbib}

\usepackage[preprint]{neurips_2022}

\usepackage[utf8]{inputenc} 
\usepackage[T1]{fontenc}    
\usepackage{hyperref}       
\usepackage{url}            
\usepackage{booktabs}       
\usepackage{amsfonts}       
\usepackage{nicefrac}       
\usepackage{microtype}      
\usepackage{xcolor}         
\usepackage{graphicx}
\usepackage{subfigure}
\usepackage{eucal}
\usepackage{wrapfig}

\usepackage{amssymb}
\usepackage{multirow}
 \usepackage{arydshln} 
\usepackage{booktabs} 
\usepackage{amsmath}
\usepackage{microtype}
\usepackage{enumitem}
\usepackage{colortbl}
\usepackage[disable]{todonotes}
\usepackage{rotating}
\usepackage{tabulary}
\usepackage{caption}
\usepackage{mathtools}

\usepackage{xparse}

\newtheorem{theorem}{Theorem}

\usepackage{enumitem}
\usepackage{amsmath}
\usepackage{mathabx}
\usepackage{amssymb}

\newtheorem{proof}{Proof}

\usepackage{ifpdf}

\title{{Diving into Unified Data-Model Sparsity for Class-Imbalanced Graph Representation Learning}}
\author{
    Chunhui Zhang\textsuperscript{\rm 1}, Chao Huang\textsuperscript{\rm 2}, Yijun Tian\textsuperscript{\rm 3},  Qianlong Wen\textsuperscript{\rm 3}, \\ \textbf{Zhongyu Ouyang\textsuperscript{\rm 3}}, \textbf{Youhuan Li\textsuperscript{\rm 4}}, \textbf{Yanfang Ye\textsuperscript{\rm 3}}, \textbf{Chuxu Zhang\textsuperscript{\rm 1}}
    \\
    \textsuperscript{\rm 1}Brandeis University, USA,
    \textsuperscript{\rm 2}University of Hong Kong, China \\
    \textsuperscript{\rm 3}University of Notre Dame, USA, 
    \textsuperscript{\rm 4}Hunan University, China \\ 
    \texttt{\{chunhuizhang}, \texttt{chuxuzhang\}@brandeis.edu}
    \texttt{chaohuang75@gmail.com}, \\
    \texttt{\{ytian5}, \texttt{qwen}, \texttt{zouyang2},\texttt{yye7\}@nd.edu}, 
    \texttt{liyouhuan@hnu.edu.cn}
}

\begin{document}
\maketitle

\begin{abstract}
Even pruned by the state-of-the-art network compression methods, recent research shows that deep learning model training still suffers from {the demand of massive data usage.} In particular, {Graph Neural Networks (GNNs) training upon such non-Euclidean graph data often encounters relatively higher time costs,} {due to its irregular and nasty density properties,} compared with data in the regular Euclidean space (e.g., image or text).
{Another natural property concomitantly with graph is class-imbalance which cannot be alleviated by the massive graph data while hindering GNNs' generalization.}
To fully tackle
these unpleasant properties, \underline{\textit{(i) theoretically}}, we introduce a hypothesis about what extent a subset {of the training} data can approximate the full {dataset}'s learning effectiveness. {The effectiveness is further} guaranteed {and proved} by the gradients' distance between the subset and the full set; \underline{\textit{(ii) empirically}}, 
{we discover that during the learning process of a GNN, some samples in the training dataset are informative for providing gradients to update model parameters.}
Moreover, the informative subset is not fixed during training process. Samples that are informative in the current training epoch may not be so in the next one.
{We refer to {this} observation {as dynamic} data sparsity.} We also notice that sparse subnets {pruned from a well-trained GNN sometimes forget the information provided by the informative subset, reflected in their poor performances upon the subset.}
{Based on these findings, we develop a unified data-model dynamic sparsity framework named \textit{\textbf{Graph Dec}antation} (GraphDec) to address {challenges brought by} training upon a massive class-imbalanced graph data.
The key idea of GraphDec is to identify {the informative subset dynamically during the training process
by adopting} sparse graph contrastive learning.}
Extensive experiments on multiple benchmark datasets demonstrate that GraphDec outperforms state-of-the-art baselines for {class-imbalanced} graph classification and {class-imbalanced} node classification tasks, {with respect to} classification accuracy and data usage {efficiency.}

\end{abstract}

\vspace{-0.1in}
\section{Introduction}
\vspace{-0.1in}
Graph representation learning (GRL)~\cite{kipf2017semi} has shown remarkable power in dealing with non-Euclidean structure data (e.g., social networks, biochemical molecules, knowledge graphs). Graph neural networks (GNNs)~\cite{kipf2017semi, hamilton2017inductive, gat2018graph,zhang2019hetgnn}, as the current state-of-the-art of GRL, have become essential {in} various graph mining applications. To learn the representation of each node reflecting its local structure pattern, 
{GNNs gather features of the neighbor nodes and} apply message passing along edges.
This topology-aware mechanism enables GNNs to achieve superior performances {over different tasks.}

However, in many real-world scenarios, 
graph data often {preserves two properties}: massiveness~\cite{thakoor2021large, hu2020ogb} and {class-imbalance}~\cite{park2022graphens}.
{Firstly, message-passing over nodes with high degrees brings about heavy computation burdens. Some of the calculations are even redundant, in that not all neighbors are informative for learning task-related embeddings. Unlike regular data such as images or texts, the connectivity of irregular graph data causes random memory access, which further slows down the efficiency of data readout.}
Secondly, class imbalance naturally exists in datasets from diverse practical {domains}, such as bioinformatics and {social} networks. GNNs are {sensitive} to this imbalance and can be biased toward the dominant classes. {This bias may mislead GNNs' learning process, therefore making the model underfit on samples that are of real importance with respect to the downstream tasks, and as a result yielding poor performance on the test data.}

\begin{wrapfigure}{r}{3in}
            \vspace{-0mm}
            \includegraphics[scale=0.6]{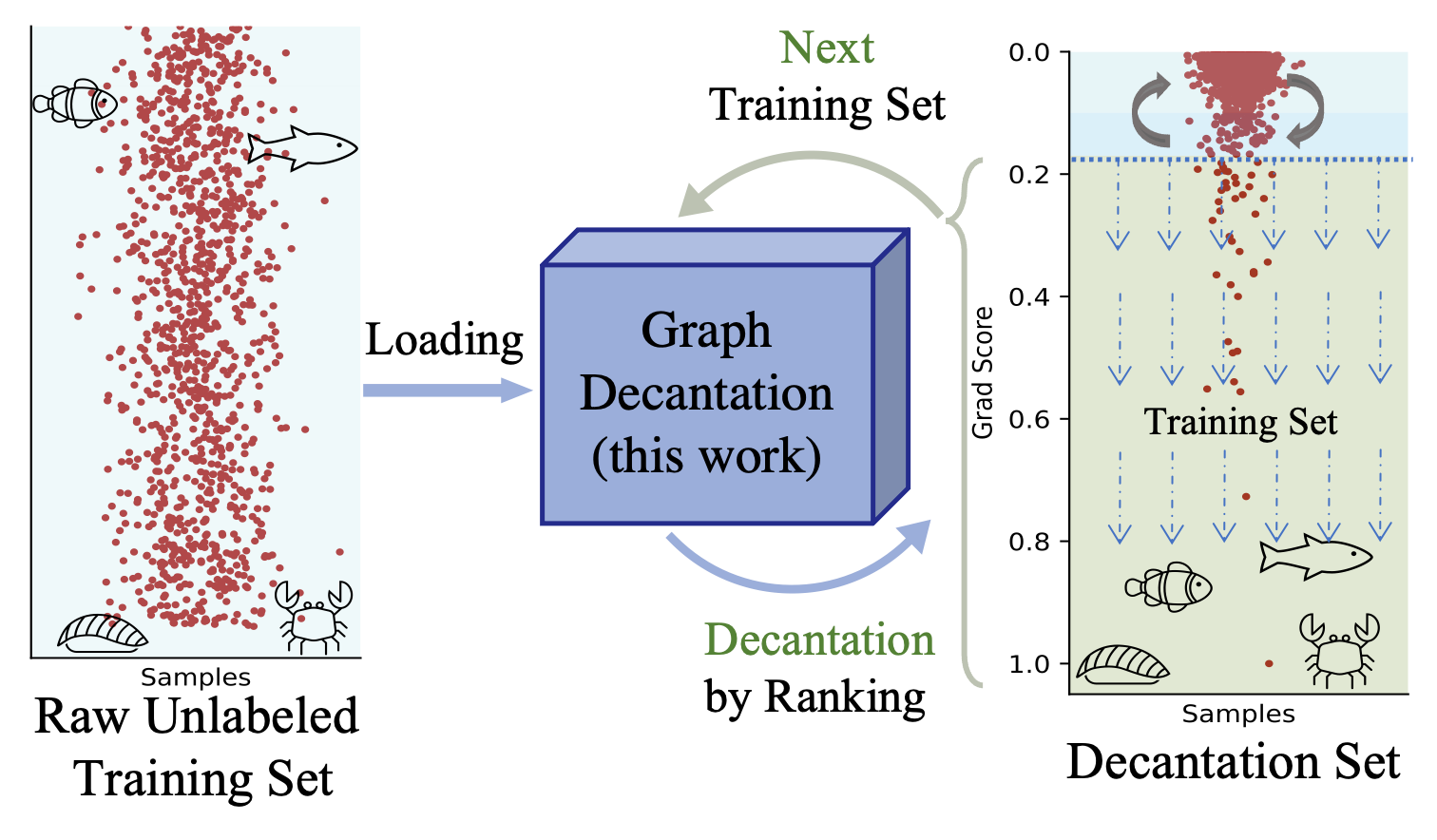}
            \vspace{-3mm}
    \caption{The principle of graph decantation. It decants data samples based on rankings of their gradient scores, and then uses them as the training set in the next epoch.}
            \vspace{-4mm}
\label{fig: decantation_intro}
\end{wrapfigure}
Accordingly, recent studies~\cite{chen2021unified, zhao2021graphsmote, park2022graphens} {arise} to address {the issues of massiveness} or {class-imbalanced} {in} graph data. To {tackle the massiveness issue}, \cite{10.1145/3178876.3186111, chen2018fastgcn} explore efficient data sampling {policies to reduce} the computational cost from {the} data perspective.
From the model improvement perspective, some approaches design the quantization-aware training and low-precision inference method to reduce GNNs' operating costs.
{For example}, GLT~\cite{chen2021unified} applies the lottery ticket technique~\cite{frankle2018the} to simplify graph data and GNN model simultaneously. 
{To deal with the imbalance issue in node classification on graphs, GraphSMOTE~\cite{zhao2021graphsmote} tries to generate new nodes for the minority classes to balance the training data.}
{Improved upon GraphSMOTE, GraphENS~\cite{park2022graphens} further proposes a new augmentation method by constructing an ego network to learn the representations of the minority classes. Despite progress made so far, existing methods fail to tackle the two issues altogether. Furthermore, while one of the issues is being handled, extra computation costs are introduced at the same time.} For example, 
{the} rewind steps in GLT~\cite{chen2021unified} {which search for lottery subnets and subsets heavily increase the computation cost,} although the final {lotteries are} lightweight.
The newly synthetic nodes in GraphSMOTE~\cite{zhao2021graphsmote, chawla2002smote} and GraphENS~\cite{park2022graphens}, {although help alleviate the data imbalance, bring extra computational burdens for the next-coming training process.}
{Regarding the issues above, we make several observations.
Firstly, we notice} that a sparse pruned GNN {easily "forgets" the minority samples when trained with class-imbalanced graph data, as it yields worse performance over the minorities compared with the original GNN~\cite{hooker2019compressed}. To investigate the cause of the above observation, we study how each graph sample affects the model parameters update process by taking a closer look at the gradients it brings to the parameters.}
Specifically, at early training stages, {we found a small subset of the samples providing the most informative supervisory} signals reflected by the gradient norms. {One hypothesis we make is that the training effectiveness of the full training set can be approximated, to some extent, by that of the subset. We further hypothesize that this effective approximation} is guaranteed by the distance between the gradients of the subset and the full dataset.

Based on the above {observations} and the hypothesises, we propose a novel method called \textbf{Graph Dec}antation (GraphDec) to explore dynamic sparsity training from both {model and data aspects}. The principle behind our method is illustrated in Figure \ref{fig: decantation_intro}. 
{Given that informative samples bring about higher gradient values/scores when trained with a sparse GNN,}
{our method is inspired by contrastive self-supervised learning because it can be modified to dynamically prune one branch of contrastive backbone for improving its capability of identifying minority samples in class-imbalanced dataset. In particular, we design a new contrastive backbone with a sparse GNN and enable the model to identify informative samples in a self-supervised manner. To the best of our knowledge, other learning processes (e.g., graph auto-encoder, supervised learning) are either unable to identify informative samples or incapable of learning in a self-supervised manner. Accordingly, the proposed framework can score samples in the current training set and keep only $k$ most informative samples as training set for the next epoch.}
{Considering that a currently unimportant sample does not imply that it will always be unimportant, 
{we further design a data recycling process} to randomly recycle prior discarded data samples {(samples that are considered unimportant in previous training epochs)}, and add them back to current informative sparse subsets for reuse.} {The dynamically} updated informative subset {supports the} sparse GNN  {to learn} more balanced {representations.} To summarize, our major contributions in this work are:
\begin{itemize}[leftmargin=*]
\vspace{-0.05in}
\item We develop a novel framework, Graph Decantation, {which} leverages dynamical sparse graph contrastive learning {on} class-imbalanced graph data with {efficient} data usage. To {our best} knowledge, this is the first study to explore the dynamic sparsity property for {class-imbalanced} graphs.
\vspace{-0.05in}
\item We introduce cosine annealing to dynamically control the sizes of the sparse GNN model and graph data subset{ to smooth} the training process. Meanwhile, we introduce data recycling to refresh the current data subset {to} avoid overfitting.
\vspace{-0.05in}
\item Comprehensive experiments on multiple benchmark datasets demonstrate that GraphDec outperforms state-of-the-art methods for both {class-imbalanced} graph classification and {class-imbalanced} node classification tasks. Additional results show that GraphDec dynamically finds {an} informative subset across different training epochs effectively. 
\end{itemize}

\vspace{-0.2in}
\section{Related Work}
\vspace{-0.1in}
\textbf{Training deep model with sparsity.}
{Despite the fact that deep neural networks work generally well in practice, they are usually over-parameterized. Over-parameterized models, although} usually achieve good performance {when trained properly, are usually associated with enormous computational cost}.
Therefore, parameter pruning {aiming at decreasing computational cost} has been a popular topic and many parameter-pruning strategies {are proposed to} balance the trade-off between model performance and {learning} efficiency \cite{deng2020model, liu2018rethinking}.
Among all of the existing parameter pruning methods, most of them belong to the static pruning {category} and deep neural networks are pruned either by neurons~\cite{han2015learning, han2015deep_compression} or {architectures} (layer and filter)~\cite{he2017channel, dong2017learning}. Parameters deleted by these methods will not be recovered later. 
In contrast to static pruning methods, more recent works propose dynamic pruning strategies {where} different compact subnets will be dynamically activated at each training iteration~\cite{mocanu2018scalable, mostafa2019parameter, raihan2020sparse}. 
The other line of computation cost reduction lies in the dataset sparsity{~\cite{karnin2019discrepancy, mirzasoleiman2020coresets, NEURIPS2021_ac56f8fe}}. The core idea is to prune the original dataset and {filter out} the most salient subset so that an over-parameterized deep model could be trained {upon (e.g.,} {data diet subset~\cite{NEURIPS2021_ac56f8fe}}). 
{Recently, the property of sparsity is also used to improve model robustness~\cite{chen2022sparsity, pmlr-v139-fu21c}.}
In this work, we attempt to {accomplish dynamic sparsity from both {the} GNN model {and the} graph dataset simultaneously.}
\\
\textbf{Class-imbalanced learning on graphs.}
In real-world scenarios, imbalanced class distribution is one of the natural properties in many datasets, including graph data. Except for conventional re-balanced methods, like reweighting samples~\cite{zhao2021graphsmote, park2022graphens} and oversampling~\cite{zhao2021graphsmote, park2022graphens}, different methods have been proposed to solve the class imbalance issue in graph data given {a} specific task.
For {the} node classification task, an early work~\cite{zhou2018sparc} tries to accurately characterize the rare categories through a curriculum self-paced strategy while some other previous works~\cite{ijcai2020-398, zhao2021graphsmote, park2022graphens} solve the {class-imbalanced} issue by proposing different methods to generate {synthetic samples} to rebalance the dataset. Compared to the node-level task, the re-balanced methods specific to graph-level task are relatively unexplored. A recent work~\cite{wang2021imbalanced} proposes to utilize additional supervisory signals from neighboring graphs to alleviate the {class-imbalanced} problem {for a} graph-level task. To the best of our knowledge, our proposed GraphDec is the first work to solve the {class-imbalanced} for both node-level and graph-level tasks.

\vspace{-0.15in}
\section{Preliminary}
\vspace{-0.1in}
In this work, we denote graph as $G = (V, E, X)$, where $V$ is the set of nodes, $E$ is the set of edges, and $X \in \mathbb{R}^{d}$ represents the node (and edge) attributes of dimension $d$. In addition, we represent the neighbor set of node $v \in V$ as $N_{v}$. 

\textbf{Graph Neural Networks.} 
GNNs~\cite{wu2020comprehensive} learn node representations from {the} graph structure and node attributes. This process can be formulated as: 
\begin{equation}
h_{v}^{(l)}=\mathrm{COMBINE}^{(l)}\left(h_{v}^{(l-1)}, \mathrm{AGGREGATE}^{(l)}\left(\left\{h_{u}^{(l-1)}, \forall u \in N_{v}\right\}\right)\right),
\label{eq: GNN}
\end{equation}
where $h_v^{(l)}$ denotes feature of node $v$ at $l$-th GNN layer;  $\mathrm{AGGREGATE}(\cdot)$ and $\mathrm{COMBINE}(\cdot)$ are neighbor aggregation and combination functions, respectively; $h_v^{(0)}$ is initialized with node attribute $X_{v}$. We obtain the output representation of each node after repeating the process in Equation (\ref{eq: GNN}) {for} $L$ rounds.
The representation of the whole graph, denoted as $h_G \in \mathbb{R}^{d}$, can be obtained by using a {READOUT} function to combine the final node {representations} learned above:
\begin{equation}
    h_G = \mathrm{READOUT}\left\{h_{v}^{(L)}~|~\forall v \in V\right\},
\end{equation}
where the $\mathrm{READOUT}$ function can be any permutation invariant, like summation, averaging, etc. 
\\
\textbf{Graph Contrastive Learning.} 
Given a graph dataset $\mathcal{D}=\left\{G_{i} \right\}_{i=1}^{N}$, Graph Contrastive Learning (GCL) methods firstly implement proper transformations on each graph $G_i$ to generate two views $G_i^{\prime}$ and $G_i^{\prime\prime}$.
{The goal of GCL is to map samples within positive pairs closer in the hidden space, while those of the negative pairs are further. GCL methods are usually optimized by contrastive loss.}
Taking the most popular InfoNCE loss~\cite{InfoNCE_LB} as an example, the contrastive loss is defined as: 
\begin{equation}
\mathcal{L}_{CL}(G_i^{\prime}, G_i^{\prime\prime})=-  \log \frac{\exp \left(\operatorname{sim}\left(\mathbf{z}_{i, 1},\mathbf{z}_{i, 2}\right)\right)}{\sum_{j=1, j \neq i}^{N} \exp \left(\operatorname{sim}\left(\mathbf{z}_{i, 1},\mathbf{z}_{j, 2}\right)\right)},
\label{eq: CL}
\end{equation}
where $\mathbf{z}_{i, 1} = f_{\theta}\left(G_i^{\prime}\right)$, $\mathbf{z}_{i,2} = f_{\theta}\left(G_i^{\prime\prime}\right)$, and $\operatorname{sim}$ denotes the similarity function.
\\
\textbf{Network Pruning.} Given {an} over-parameterized deep neural network $f_{\theta}(\cdot)$ with weights ${\theta}$, the network pruning is usually {performed layer-by-layer}. {The} pruning process of the $l_{th}$ layer in $f_{\theta}(\cdot)$ can be formulated as follows: 
\begin{equation}
    \theta_{pruned}^{l_{th}} = \mathrm{TopK}(\theta^{l_{th}}, k), k=\alpha \times {|}\theta^{l_{th}}|,
\label{eq:pruning-model}
\end{equation}
where $\theta^{l_{th}}$ is the parameters in the $l_{th}$ layer of $f_{\theta}(\cdot)$ and 
$\mathrm{TopK}(\cdot, k)$ refers to the operation to choose the top-$k$ largest elements of $\theta^{l_{th}}$. 
We use a pre-defined sparse rate $\alpha$ to control the fraction of parameters kept in the pruned network $\theta_{pruned}^{l_{th}}$. Finally, only the top $k=\alpha \times {|}\theta^{l_{th}}|$ largest weights will be kept in the pruned layer. The pruning process will be implemented iteratively to prune the parameters in each layer of deep neural network~\cite{han2015deep}.

\vspace{-0.15in}
\section{Methodology}
\vspace{-0.1in}
In this section, we first illustrate our sparse subset approximation hypothesis supported {by the theorem, which means that if the gradients} of a data subset approximate well to the {gradients} of the full data set, the model trained on subset performs closely to the model trained with full set. Guided by this hypothesis, we develop GraphDec to continually refine a compact training subset with {the dynamic graph contrastive learning methodology}. In detail, we describe procedures about how to rank the importance of each sample, smooth the refining procedure, and avoid overfitting. 

\vspace{-0.1in}
\subsection{Sparse Subset Approximation Hypothesis} 
\vspace{-0.1in}
\label{sec:Hypothesis}
{Firstly,} we propose the sparse subset approximation hypothesis to show how {a} model trained with a subset data $\mathcal{D}_{S}$ can approximate the effect of {a} model trained with full data $\mathcal{D}$. 
{This hypothesis explains} why the performance of the model trained with a subset data selected by specific methods {(e.g., data diet~\cite{NEURIPS2021_ac56f8fe}) achieves performance close to the one trained on the full dataset.}

\begin{theorem}
\label{thm}

For a data selection algorithm, we assume the model is optimized with full gradient descent. 
{At epoch $t$ where $t \in \left[1, T\right]$}, denote the model's parameters as $\theta^{(t)}$ {where ${\left\Vert\theta^{(t)}\right\Vert}^2 \leq {d}^2$ and} d is constant, the optimal model's parameters as $\theta^*$, subset data as $\mathcal{D}_{S}^{(t)}$, {and} learning rate as $\alpha$. {Define} gradient error  
$\mbox{Err}(\mathcal{D}_{S}^{(t)}, \mathcal{L}, \mathcal{L}_{train}, \theta^{(t)}) = {\left\Vert \sum_{i \in \mathcal{D}_{S}^{(t)}} \nabla_{\theta}\mathcal{L}_{train}^i(\theta^{(t)}) -  \nabla_{\theta}\mathcal{L}(\theta^{(t)})\right\Vert}$, where 
$\mathcal{L}$ denotes training loss $\mathcal{L}_{train}$ over full training data or validation loss $\mathcal{L}_{val}$ over full validation data{.} $\mathcal{L}$ is a convex function. Then we have the following guarantee:

If $\mathcal{L}_{train}$ is Lipschitz continuous with parameter $\sigma_T$ and $\alpha = \frac{d}{\sigma_T \sqrt{T}}$, then $\min_{t = 1:T}
\mathcal{L}(\theta^{(t)}) 
- \mathcal{L}(\theta^*) \leq \frac{d\sigma_T}{\sqrt{T}} + \frac{d}{T}\sum_{t=1}^{T-1} \mbox{Err}(\mathcal{D}_{S}^{(t)}, \mathcal{L}, \mathcal{L}_{train}, \theta^{(t)})$. \\
\end{theorem}
\vspace{-0.1in}
The detailed proof is provided in the Section A of Appendix. 
According to the above hypothesis, one intuitive illumination is that reducing the distance between gradients {of} the subset and the full set, formulated as ${\left\Vert \sum_{i \in \mathcal{D}_{S}^{(t)}} \nabla_{\theta}\mathcal{L}_{train}^i(\theta^{(t)}) -  \nabla_{\theta}\mathcal{L}(\theta^{(t)})\right\Vert}$, is the key to minimize the gap between {the} performance of the model trained with the subset and the optimal model, denoted as $\mathcal{L}(\theta) - \mathcal{L}(\theta^*)$.
From the perspective of minimizing ${\left\Vert \sum_{i \in \mathcal{D}_{S}^{(t)}} \nabla_{\theta}\mathcal{L}_{train}^i(\theta^{(t)}) - \nabla_{\theta}\mathcal{L}(\theta^{(t)})\right\Vert}$, the success of {data diet~\cite{NEURIPS2021_ac56f8fe}} (a prior coreset algorithm) is understandable: data diet computes each sample's error/gradient norm based on a slight-trained model, {then deletes a portion of {the full set} with smaller values, which can be represented as $\bar{\mathcal{D}_{S}}^{(t)} = {D} - {\mathcal{D}_{S}}^{(t)}$.} 
The gradient $\sum_{j \in \bar{\mathcal{D}_{S}}^{(t)}} \nabla_{\theta}\mathcal{L}_{train}^j(\theta^{(t)})$ of {the} removed data samples is much smaller than that of {the} remaining data samples $\sum_{i \in \mathcal{D}_{S}^{(t)}} \nabla_{\theta}\mathcal{L}_{train}^i(\theta^{(t)})$. 
As we will show in the experiments (Section~\ref{subsec:evo-subset}), the static data diet cannot always {capture} the most important {samples} across all epochs during training~\cite{NEURIPS2021_ac56f8fe}. Although rankings of all elements in $\mathcal{D}_{S}$ seemly keep static and unchangeable, the ranking order of elements in full training dataset $\mathcal{D}$ changes much more actively than {the} diet subset $\mathcal{D}_{S}$, which implies one-shot subset $\mathcal{D}_{S}$ can not provide gradient $\sum_{i \in \mathcal{D}_{S}^{(t)}} \nabla_{\theta}\mathcal{L}_{train}^i(\theta^{(t)})$ to approximate the full set $\mathcal{D}$'s gradient $\nabla_{\theta}\mathcal{L}(\theta^{(t)})$.

\begin{figure*}[th]
    \centering
    {\includegraphics[width=1\textwidth]{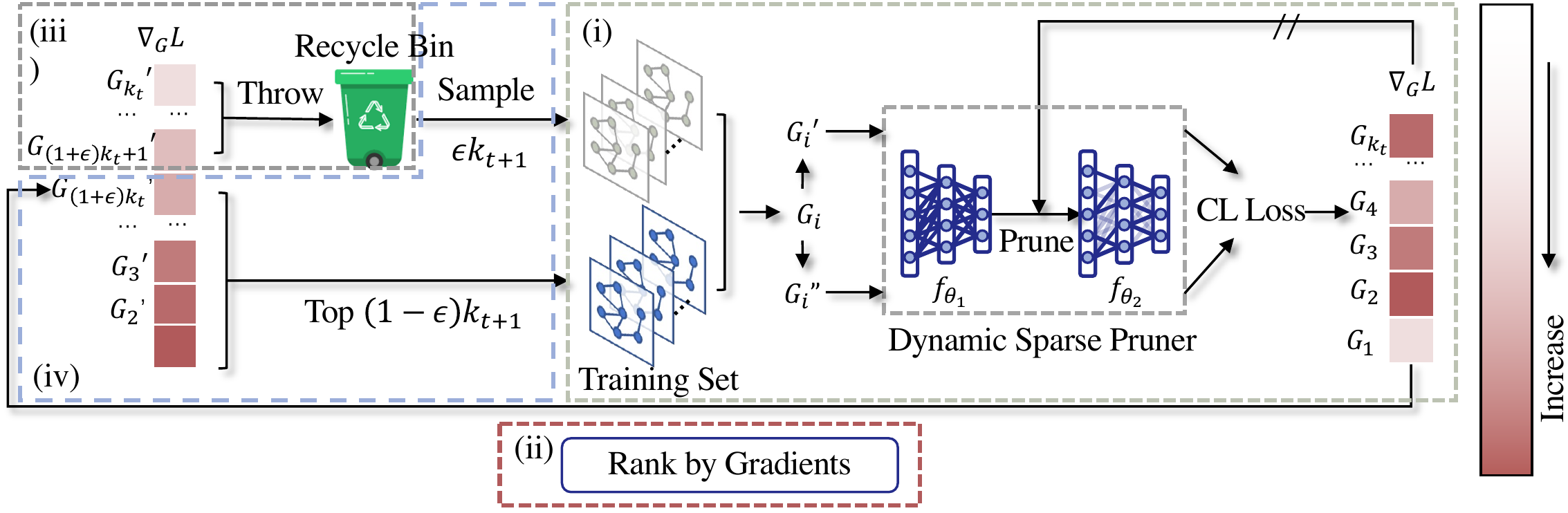}}
    \vspace{-0.2in}
    \caption{The overall framework of GraphDec: 
    (i) The dynamic sparse graph contrastive learning model computes gradients for graph/node samples; 
    (ii) The input samples are sorted according to their gradients;
    (iii) Part of samples with the smallest gradients are thrown into the recycling bin;
    (iv) Part of samples with the largest gradients in the current epoch and some sampled randomly from the recycling bin are jointly used as training input in the next epoch.}
     \vspace{-0.2in}
\label{fig:framework}
\end{figure*}

\vspace{-0.1in}
\subsection{Graph Decantation}
\vspace{-0.1in}
Then, we follow the above Theorem~\ref{thm} to develop GraphDec for achieving competitive performance and efficient data usage simultaneously by {filtering out} the most influential data subset.
The overall framework of GraphDec is illustrated in Figure~\ref{fig:framework}. {The training processes are summarized} into four steps:
(i) {First,} compute gradients of all $M^{(t)}$ graph{/node} samples in $t$-th epoch from contrastive learning loss; 
(ii) The gradient{s are} then normalized and {the} corresponding graph/node samples are ranked in a descending order by {their} magnitudes;
(iii) {We then} decay the number of samples from $M^{(t)}$ to $M^{(t+1)}$ with cosine annealing and only keep the top $M^{(t+1)} \times (1 - \epsilon)$ samples ($\epsilon$ {is the} exploration rate which controls the ratio of samples randomly resampled from recycle bin){. The} rest samples will be {throwed into the} recycle bin temporarily;
(iv) Finally, randomly resample $M^{(t+1)} \times \epsilon$ samples from the recycled bin and {these samples union the ones} selected in step (iii) {will} be used for model training in the $t+1$ epoch.
In the following, we describe each of these four steps {in details.}

\vspace{-0.1in}
\paragraph{Compute gradients by dynamic sparse graph contrastive learning model.}
In the first step, given a graph training set $\mathcal{D}=\left\{G_{i} \right\}_{i=1}^{N}$ as input, our dynamic sparse graph contrastive learning model (DS-GCL) {takes} two augmented views $G^{\prime}$ and $G^{\prime\prime}$ of an original graph $G \in \mathcal{D}$ as {inputs.} In detail, for each graph sample, DS-GCL has two GNN branches $f_{\theta_{1}}(\cdot)$ and  $f_{\theta_{2}}(\cdot)$, which are pruned on-the-fly from an original GCN $f_\theta(\cdot)$ by a dynamic sparse pruner. For example, at $l_{th}$ graph convolutional layer of $f_\theta(\cdot)$, a fraction of connections with the largest weight magnitudes are kept, which are chosen by the following formulation:
\begin{equation}
    \theta_{pruned}^{l_{th}} = \mathrm{TopK}(\theta^{l_{th}}, k), k=\alpha^{(t)} \times {|}\theta^{l_{th}}|,
\label{eq:topk-model}
\end{equation}
where $\alpha^{(t)}$ is the fraction of remaining neural connections, which is controlled under cosine annealing:
\begin{equation}
    \alpha^{(t)} = 
    \frac{\alpha^{(0)}}{2}
    \left\{1 + \cos(\frac{\pi t}{T})\right\}, 
    t \in \left[1, T\right],
\label{eq:cosine-schedule-model}
\end{equation}
where $\alpha^{(0)}$ is initialized as 1. In addition, some new connections are activated using the current gradient information. {Every few epochs}, the pruned neural connections are all re-involved in loss backward by following formulation: 
\begin{equation}
    \mathbb{I}_{\theta^{l_{th}}} = \mathrm{ArgTopK}(\nabla_{\theta^{l_{th}}}\mathcal{L}, k), k=\alpha^{(t)} \times {|}\theta^{l_{th}}|
    ,
\label{eq:topk-model-rewind}
\end{equation}
where $\mathrm{ArgTopK}$ returns indices of top-$k$ elements and $\mathbb{I}_{\theta_{pruned}^{l_{th}}}$ denotes elements' indices in $l_{th}$ layer weights ${\theta^{l_{th}}}$. These reactivated weights are then combined with other remaining connections for model pruning {in the} next iteration.
We save {the} gradient values for all samples and use them in {the} next step.
The benefits brought from DS-GCL {reflects in} two {perspectives}: (a) it scores {the} graph samples without {any labeling} effort {from humans, compared} with graph active learning; (b) it is more sensitive {in selecting} informative samples, empirically verified in the Section C of appendix.

\vspace{-0.1in}
\paragraph{Rank graph samples according to their gradients' $L_{2}$ norms.}
In the second step, since gradients of all graph samples in $\mathcal{D}_{S}^{(t)}$ ($\mathcal{D}_{S}^{(t)} = \mathcal{D}$ when $t$ = 0) at $t$-th epoch are already saved, we can calculate their gradients' $L_{2}$ norms. For example, a graph input ${G}_{i} \in \mathcal{D}_{S}^{(t)}$ will be scored by its gradient norm:
\begin{equation}
    g({x}^{(t)}) = \left\|  \nabla_{f_{{\theta}_{pruned}}} \mathcal{L} (f_{P{\theta}_{pruned}(G^{\prime}}), f_{{\theta}_{pruned}}({G^{\prime\prime}})) \right\|_2.
\label{eq:gradient-norm}
\end{equation}
In this work, we use the popular InfoNCE contrastive loss~\cite{van2018representation} and the gradient of $G$ is computed as: 
\begin{equation}
    \nabla_{f_{{\theta}_{pruned}}} \mathcal{L} (f_{{\theta}_{pruned}}(G^{\prime}), f_{{\theta}_{pruned}(G^{\prime\prime}})) = p({\theta_{pruned}}, G^{\prime}) - p({\theta_{pruned}}, G^{\prime\prime}),
\label{eq:gradient-compute}
\end{equation}
where the $p({\theta_{pruned}}, G^{\prime})$ and $p({\theta_{pruned}}, G^{\prime\prime})$ denote model's predictions of $G^{\prime}$ and $G^{\prime\prime}$ with pruned parameters ${\theta_{pruned}}$. All graph samples in $\mathcal{D}_{S}^{(t)}$ are ranked according to their scores. The ranked $\mathcal{D}_{S}^{(t)}$ will be provided to latter use.

\paragraph{Decay the size of $\mathcal{D}_{S}$ by cosine annealing.} In the third step, we aim to prune the size of the subset for the next $t+1$ training epoch. To smooth this pruning procedure, we apply cosine annealing to control the decay rate. Specifically, the size $M^{(t+1)}$ is computed as follows:
\begin{equation}
    M^{(t+1)} = 
    \frac{M^{(0)}}{2}
    \left\{1 + \cos(\frac{\pi (t+1)}{T})\right\}, t \in \left[1, T\right]
    .
\label{eq:cosine-schedule-data}
\end{equation}

It smoothly refines $\mathcal{D}_{S}$ and avoids manually choosing which training epoch for one-shot selection like {data diet~\cite{NEURIPS2021_ac56f8fe}}. $M^{(t+1)}$ sets the number of graph samples in $\mathcal{D}_{S}^{(t+1)}$ for the next $t+1$ epoch.

As we will show in Figure~\ref{fig:evolution} in the experiments, 
at early training, some graph samples only have {low scores/importance}. However, in {the} later {training} epochs, these graph samples {yield} much higher scores {once given more patience in training}. {Upon} this observation, we {believe} that it {is worthwhile to not permanently remove} samples with low scores at the current training epoch, since some samples in removal set $\bar{\mathcal{D}_{S}} = \mathcal{D} - \mathcal{D}_{S}^{(t)}$ might 
be re-identified as {high-scored} samples if they can be re-involved {into} the training {process.}
In the opposite direction, if {a} model is {only} trained with a subset {of} graph samples {that are highly scored} in the early training {stage, the} training effect {of such a model} cannot approximate the full training set's gradient effects well. 
{Based} on this analysis, this step specializes in this dilemma: we build the above cosine annealing to control the removal rate of $\mathcal{D}_{S}^{(t)}$ during training instead of hastily scoring out a subset in one-shot mode like data diet and then use it to re-train the neural network model.

\vspace{-0.1in}
\paragraph{Recycle removed graph samples for next training epoch.}
In the last step, we already have the ranked $\mathcal{D}_{S}^{(t)}$ and the subset size $M^{(t+1)}$ for $t+1$ epoch. {Our next} aim {is} to update the elements in $\mathcal{D}_{S}^{(t+1)}$ for {the} next epoch. 
When updating elements in $\mathcal{D}_{S}^{(t+1)}$, {since we think currently {low-scored samples may still have the} potential to be high-scored, {removed samples are randomly recovered.}} We use an exploration rate $\epsilon$ to remove $\epsilon M^{(t+1)}$ lowest-scores graph samples in $\mathcal{D}_{S}^{(t)}$ and recycles $\epsilon M^{(t+1)}$ samples from $\bar{\mathcal{D}_{S}}^{(t-1)}$. At the same time, we keep $(1-\epsilon) M^{(t+1)}$ graph samples with highest scores from $\mathcal{D}_{S}^{(t)}$ to $\mathcal{D}_{S}^{(t+1)}$. The overall $\mathcal{D}_{S}^{(t+1)}$'s update is worked as follows:
\begin{equation}
    \mathcal{D}_{S}^{(t+1)} = {\mathrm{TopK}(\mathcal{D}_{S}^{(t)}, (1-\epsilon) M^{(t+1)}) \bigcup \mathrm{SampleK}(\bar{\mathcal{D}_{S}}^{(t-1)}, \epsilon M^{(t+1)})},
\label{eq:topk-data-recycling}
\end{equation}
where $\mathrm{SampleK}(\bar{\mathcal{D}_{S}}^{(t-1)}, \epsilon M^{(t+1)})$ returns randomly sampled $\epsilon M^{(t+1)}$ samples from $\bar{\mathcal{D}_{S}}^{(t-1)}$.
Given the compact sparse subset $\mathcal{D}_{S}^{(t+1)}$, we use it for model training in the next epoch and repeat to execute this pipeline until $T$ epoch.

\vspace{-0.15in}
\section{Experiments}
\vspace{-0.1in}
In this section, we conduct extensive experiments to validate the effectiveness of the proposed model for both {class-imbalanced} graph classification and {class-imbalanced} node classification tasks. We also conduct ablation study and informative subset evolution analysis to better understand the effectiveness of the proposed model. Due to space limit, more analyses about GraphDec property are provided in Section C of Appendix. 
 
 \vspace{-0.1in}
\subsection{Experimental Setup}
\label{sec:exp-set}
\vspace{-0.1in}
\textbf{Datasets}. 
We use various graph benchmark datasets to evaluate our model for two tasks: graph classification and node classification in {class-imbalanced} data scenario.
For the {class-imbalanced} graph classification task, we consider all seven datasets used in G$^2$GNN paper~\cite{wang2021imbalanced}, i.e., MUTAG, PROTEINS, D\&D, NCI1, PTC-MR, DHFR, and REDDIT-B in ~\cite{morris2020tudataset}. For the {class-imbalanced} node classification task, we use all five datasets used in the GraphENS paper \cite{park2022graphens}, i.e., Cora-LT, CiteSeer-LT, PubMed-LT~\cite{sen2008collective}, Amazon-Photo, and Amazon-Computers.
Detailed descriptions of these datasets are provided in the Section B of Appendix. 
\\
\textbf{Baseline Methods}. 
We compare our model with a variety of baseline methods using different rebalance methods. For {class-imbalanced} graph classification, we consider three rebalance methods, i.e., vanilla {(without re-balancing when training)}, up-sampling \cite{wang2021imbalanced}, and re-weight \cite{wang2021imbalanced}.  For each rebalance method, we run three baseline methods including GIN \cite{xu2018how}, InfoGraph \cite{sun2019infograph}, and GraphCL \cite{You2020GraphCL}. In addition, we adopt two versions of G$^2$GNN (i.e., remove-edge and mask-node) \cite{wang2021imbalanced} for in-depth comparison.
For {class-imbalanced} node classification, we consider nine baseline methods including vanilla, re-weight~\cite{japkowicz2002class}, oversampling~\cite{park2022graphens}, cRT~\cite{Kang2020Decoupling}, PC Softmax~\cite{hong2021disentangling}, DR-GCN~\cite{ijcai2020-398}, GraphSMOTE~\cite{zhao2021graphsmote}, and GraphENS~\cite{park2022graphens}.
We use Graph Convolutional Network (GCN) \cite{kipf2017semi} as the default architecture for all rebalance methods.
Further details about the baselines are illustrated in the Section B of Appendix. 
\\
\textbf{Evaluation Metrics}. 
To fully evaluate the model performance, we adopt F1-micro (F1-mi.) and F1-macro (F1-ma.) scores for the {class-imbalanced} graph classification, as well as accuracy (Acc.), balanced accuracy (bAcc.), and F1-macro (F1-ma.) score for the {class-imbalanced} node classification.
\\
\textbf{Experimental Settings}. 
We adopt GCN~\cite{kipf2017semi} as the GNN backbone of GraphDec for both tasks. In particular, we use a two-layers GCN and a one-layer fully-connected layer for node classification, and add one extra average pooling operator as the readout layer for graph classification. 
We follow \cite{wang2021imbalanced} and \cite{park2022graphens} to set the imbalance ratios for graph classification and node classification tasks, respectively. In addition, we use GraphCL~\cite{You2020GraphCL} as the graph contrastive learning framework, and use cosine annealing to dynamically control the sparsity rate in the GNN model and the dataset. 
We set the initial sparsity rate the rate $\alpha^{(0)}$ for model to 0.8 and $\beta^{(0)}$ for dataset to 1.0. 
After the contrastive pre-training, we use the GCN output logits as the input to the Support Vector Machine for fine-tuning.
GraphDec is implemented in PyTorch and trained on NVIDIA V100 GPU. 

\begin{table*}[t]
\caption{{Class-imbalanced} graph classification results.
Numbers after each dataset name indicate imbalance ratios of minority to majority categories. {Best/second-best results are in bold/underline.}}
\vspace{-1mm}
\renewcommand\arraystretch{1}
\centering
\resizebox{1\textwidth}{!}{
\begin{tabular} {l|c|cc|cc|cc|cc|cc}
\toprule
{Rebalance} &\multirow{2}{*}{Basis} & \multicolumn{2}{c|}{MUTAG (5:45)} &\multicolumn{2}{c|}{PROTEINS (30:270)} &\multicolumn{2}{c|}{D\&D (30:270)} &\multicolumn{2}{c|}{NCI1 (100:900)} &\multicolumn{2}{c}{Sparsity (\%)}  \\ 
\cmidrule{3-4} \cmidrule{5-6} \cmidrule{7-8}\cmidrule{9-10} \cmidrule{11-12} Method& & F1-ma. & F1-mi. & F1-ma. & F1-mi. & F1-ma. & F1-mi. & F1-ma. & F1-mi.  & data & model  \\
\midrule
\multirow{3}{*}{vanilla} &GIN~\cite{xu2018how} &52.50 &56.77 &25.33 &28.50 &9.99 &11.88 &18.24 &18.94 & 100 & 100\\
&InfoGraph~\cite{sun2019infograph} & 69.11 & 69.68 & 35.91 & 36.81 & 21.41 & 27.68 & 33.09 & 34.03 & 100 & 100\\
&GraphCL~\cite{You2020GraphCL} & 66.82 & 67.77 & 40.86 & 41.24 & 21.02 & 26.80 & 31.02 & 31.62 & 100 & 100\\
\midrule
\multirow{3}{*}{up-sampling} &GIN~\cite{xu2018how} & 78.03 & 78.77 & 65.64 & 71.55 & 41.15 & 70.56 & 59.19 & 71.80 & >100 & 100\\
&InfoGraph~\cite{sun2019infograph} & 78.62 & 79.09 & 62.68 & 66.02 & 41.55 & 71.34 & 53.38 & 62.20 & >100 & 100\\
&GraphCL~\cite{You2020GraphCL} & 80.06 & 80.45 & 64.21  & 65.76 & 38.96 & 64.23 & 49.92 & 58.29 & >100 & 100\\
\midrule
\multirow{3}{*}{re-weight} &GIN~\cite{xu2018how} & 77.00 & 77.68 & 54.54 & 55.77 & 28.49 & 40.79 & 36.84 & 39.19 & 100 & 100\\
&InfoGraph~\cite{sun2019infograph} & {80.85} & {81.68} & 65.73 & 69.60 & 41.92 & 72.43 & 53.05 & 62.45 & 100 & 100\\
&GraphCL~\cite{You2020GraphCL} & 80.20 & 80.84 & 63.46 & 64.97 & 40.29 & 67.96 & 50.05 & 58.18 & 100 & 100\\
\midrule
\multirow{2}{*}{G$^2$GNN~\cite{wang2021imbalanced}} &remove edge & 80.37 & 81.25 & {\underline{67.70}} & {73.10} & {43.25} & {\underline{77.03}} & {63.60} & {72.97} & 100 & 100\\
&mask node & {\underline{83.01}} & {\underline{83.59}} & {67.39} & {\underline{73.30}} & {\underline{43.93}} & {\textbf{79.03}} & {\underline{64.78}} & {\underline{74.91}} & 100 & 100\\
\midrule
GraphDec &dynamic sparsity &\textbf{85.71} &\textbf{85.71} &\textbf{76.92} &\textbf{76.89} &\textbf{77.97} &\underline{77.02} &\textbf{76.30} &\textbf{76.29} & 50 & 50\\
\bottomrule
\end{tabular}}

\resizebox{1\textwidth}{!}{
\begin{tabular} {l|c|cc|cc|cc|cc|cc}
\toprule
{Rebalance} &\multirow{2}{*}{Basis} & \multicolumn{2}{c|}{{PTC-MR} (9:81)} &\multicolumn{2}{c|}{{DHFR} (12:108)} &\multicolumn{2}{c|}{{REDDIT-B} (50:450)} &\multicolumn{2}{c|}{{Avg. Rank}} &\multicolumn{2}{c}{Sparsity (\%)}  \\ 
\cmidrule{3-4} \cmidrule{5-6} \cmidrule{7-8}\cmidrule{9-10} \cmidrule{11-12} Method& & F1-ma. & F1-mi. & F1-ma. & F1-mi. & F1-ma. & F1-mi. & F1-ma. & F1-mi.  & data & model  \\
\midrule
\multirow{3}{*}{vanilla} &GIN~\cite{xu2018how} & 17.74 & 20.30 & 35.96 & 49.46 & 33.19 & 36.02 & 12.00 & 12.00 & 100 & 100\\
&InfoGraph~\cite{sun2019infograph} & 25.85 & 26.71 & 50.62 & 56.28 & 57.67 & 67.10 & 11.00 & 11.14 & 100 & 100\\
&GraphCL~\cite{You2020GraphCL} & 24.22 & 25.16 & 50.55 & 56.31 & 53.40 & 62.19 & 10.71 & 10.57 & 100 & 100\\
\midrule
\multirow{3}{*}{up-sampling} &GIN~\cite{xu2018how} & 44.78 & 55.43 & 55.96 & 59.39 & 66.71 & 83.00 & 6.00 & 5.43 & >100 & 100\\
&InfoGraph~\cite{sun2019infograph} & 44.29 & 48.91 & 59.49 & 61.62 & 67.01 & 78.68 & 6.00 & 6.00 & >100 & 100\\
&GraphCL~\cite{You2020GraphCL} & 45.12 & 53.50 & 60.29 & 61.71 & 62.01 & 75.84 & 6.29 & 6.43 & >100 & 100\\
\midrule
\multirow{3}{*}{re-weight} &GIN~\cite{xu2018how} & 36.96 & 43.09 & 55.16 & 57.78 & 45.17 & 51.92 & 9.86 & 9.86 & 100 & 100\\
&InfoGraph~\cite{sun2019infograph} & 44.09 & 49.17 & 58.67 & 60.24 & 65.79 & 77.35 & {5.43} & 5.29 & 100 & 100\\
&GraphCL~\cite{You2020GraphCL} & 44.75 & 52.22 & {60.87} & {61.93} & 62.79 & 76.15 & 6.00 & 6.29 & 100 & 100\\
\midrule
\multirow{2}{*}{G$^2$GNN~\cite{wang2021imbalanced}} &remove edge & {46.40}  & {56.61} & {\underline{61.63}} & {\underline{63.61}} & {\underline{68.39}} & {\underline{86.35}} & {2.71} & {2.86}  & 100 & 100\\
&mask node & {\underline{46.61}} & {\underline{56.70}} & 59.72 & 61.27 & {67.52} & {85.43} & {2.71}  & {2.71} & 100 & 100\\
\midrule
GraphDec &dynamic sparsity &\textbf{54.03} &\textbf{61.17} &\textbf{64.25} &\textbf{67.91} &\textbf{69.70} &\textbf{87.00} &1.00 &1.14 & 50 & 50\\
\bottomrule
\end{tabular}}
\vspace{-0.25in}
\label{tab:graph-cls-1}
\end{table*}

\vspace{-0.1in}
\subsection{{Class-imbalanced} Graph Classification Performance}
\vspace{-0.1in}
\label{sec:imbalanced_graph_classification}
We start by comparing GraphDec with the aforementioned baselines on {class-imbalanced} graph classification task. The results are reported in Table \ref{tab:graph-cls-1}. The best and second-best values are highlighted by bold and underline.
From the table, we find that GraphDec outperforms baseline methods on both metrics across different datasets, while only uses an average of 50\% data and 50\% model weights per round. Although a slight F1-micro difference has been detected on D\&D when comparing GraphDec to the best baseline G$^2$GNN, this is understandable due to the fact that the graphs in D\&D are significantly larger than those in other datasets, necessitating specialized designs for graph augmentations (e.g., the average graph size in terms of node number is 284.32 for D\&D, but 39.02 and 17.93 for PROTEINS and MUTAG, respectively).
However, in the same dataset, G$^2$GNN can only achieve 43.93 on F1-macro while GraphDec achieves 77.97, which complements the 2\% difference on F1-micro and further demonstrates GraphDec's ability to learn effectively even on the dataset with large graphs.
Specifically, models trained with vanilla setting perform the worst due to the ignorance of class imbalance. Up-sampling strategy improves the performance, but it introduces additional unnecessary data usage by sampling the minorities multiple times.
Similarly, re-weight strategy tries to address the {class-imbalanced} issue by assigning different weights to different samples. However, it still requires the label data to calculate the weight and thus may not perform well when labels are missing.
G$^2$GNN, as the best baseline, obtains decent performance by considering the usage of rich supervisory signals from both globally and locally neighboring graphs. Finally, the proposed model, GraphDec, achieves the best performance with the ability to capture dynamic sparsity on both GNN model and graph datasets. 
In addition, we rank the performance of GraphDec with regard to baseline methods on each dataset. GraphDec ranks 1.00 and 1.14 on average, which further demonstrates the superiority of GraphDec.
Noticed that all existing methods utilize the entire datasets and the model weights. However, GraphDec uses only half of the data and weights to achieve superior performance.

\vspace{-0.1in}
\subsection{{Class-imbalanced} Node Classification Performance}
\vspace{-0.1in}

To demonstrate the effectiveness of GraphDec in handling {class-imbalanced} node data, we further evaluate GraphDec in the task of {class-imbalanced} node classification.
We first evaluate GraphDec on three long-tailed citation networks (i.e., Cora-LT, CiteSeer-LT, PubMed-LT) and report the results on Table \ref{tab:node-cls-1}. We find that GraphDec obtains the best performance compared to baseline methods with different metrics. 
Specifically,
GraphSmote and GraphENS achieve satisfactory performance by generating virtual nodes to enrich the information involved in the representations of minority category.
However, GraphDec does not rely on synthetic virtual nodes to learn balanced representations, thereby avoiding the unnecessary learning costs on additional data.
Similarly to the {class-imbalanced} graph classification task in Section \ref{sec:imbalanced_graph_classification}, GraphDec leverages only half of the data and model weights, but achieves state-of-the-art performance, whereas all baselines require the full dataset and model weights but perform worse.
To validate the efficacy of the proposed model on  the real-world data, we also evaluate GraphDec on naturally class-imbalanced benchmark datasets (i.e., Amazon-Photo and Amazon-Computers).
We can see that GraphDec yields the best performance on both datasets, which further demonstrates the effectiveness of our model in handling node imbalance.

\begin{table*}[t]
\caption{{Class-imbalanced} node classification results. { Best/second-best results are in bold/underline.}
}
\vspace{-1mm}
\renewcommand\arraystretch{1}
\centering
\resizebox{1\textwidth}{!}{
\begin{tabular} {l|ccc|ccc|ccc|cc|cc|cc}
\toprule
\multirow{2}{*}{Method}  &\multicolumn{3}{c|}{Cora-LT} &\multicolumn{3}{c|}{CiteSeer-LT} & \multicolumn{3}{c|}{PubMed-LT} & \multicolumn{2}{c|}{A.P. ($\rho=$82)} &\multicolumn{2}{c|}{A.C. ($\rho=$244)} &\multicolumn{2}{c}{Sparsity (\%)}  \\ 
\cmidrule{2-16} & Acc. & bAcc. & F1-ma. & Acc. & bAcc. & F1-ma. & Acc. & bAcc. & F1-ma. & (b)Acc. & F1-ma. & (b)Acc. & F1-ma.  & data & model  \\
\midrule
vanilla      &73.66 &62.72 &63.70 &53.90 &47.32 &43.00 &70.76 &57.56 &51.88 &82.86  &78.72  &68.47  &64.01 &100 &100 \\
Re-Weight~\cite{park2022graphens}    &75.20  &68.79  &69.27  &62.56 &55.80  &53.74  &77.44  &72.80  &73.66  &92.94  &92.95  &90.04  &90.11 &100 &100\\
Oversampling~\cite{park2022graphens} &77.44  &70.73 &72.40  &62.78  &56.01  &53.99  &76.70  &68.49  &69.50  &92.46  &92.47  &89.79  &89.85 &>100 &100\\
cRT~\cite{Kang2020Decoupling}          &76.54  &69.26  &70.95  &60.60  &54.05  &52.36  &75.10  &67.52  &68.08  &91.24  &91.17  &86.02  &86.00  &100 &100\\
PC Softmax~\cite{hong2021disentangling}   &76.42  &71.30  &71.24  &65.70  &\textbf{61.54}  &\underline{61.49}  &76.92  &\underline{75.82}  &74.19  &93.32  &93.32  &86.59  &86.62 &100 &100\\
DR-GCN~\cite{ijcai2020-398}       &73.90  &64.30  &63.10  &56.18  &49.57  &44.98  &72.38  &58.86  &53.05  & N/A & N/A & N/A & N/A &100 &100\\
GraphSmote~\cite{zhao2021graphsmote}   &76.76  &69.31 &70.21  &62.58  &55.94  &54.09  &75.98  &70.96  &71.85  &92.65  &92.61  &89.31  &89.39 &>100 &100\\
GraphENS~\cite{park2022graphens}     &\underline{77.76}  &\underline{72.94}  &\underline{73.13}  &\textbf{66.92}  &60.19  &58.67  &\underline{78.12}  &74.13  &\underline{74.58}  &\underline{93.82}  &\underline{93.81}  &\underline{91.94}  &\underline{91.94}  &>100 &100\\
\midrule
GraphDec &\textbf{78.29}  &\textbf{73.94}  &\textbf{74.25}  &\textbf{66.90}  &\textbf{61.56}  &\textbf{61.85}  &\textbf{78.20}  &\textbf{76.05}  &\textbf{76.32}  &\textbf{93.85}  &\textbf{94.02}  &\textbf{92.19}  &\textbf{92.16}  &50 &50\\
\bottomrule
\end{tabular}}
\label{tab:node-cls-1}
\vspace{-0.2in}
\end{table*}

\begin{table}[t]
\centering
\caption{Ablation study results for both tasks. Four rows of red represent removing four individual components from data sparsity perspective.
Four rows of blue represent removing four individual components from model sparsity perspective.
{Best results are in bold.}
}
\vspace{1mm}
\label{tab:ablation-on-graph-sparsity}
\setlength{\tabcolsep}{1.5mm}{
\resizebox{1\textwidth}{!}{
\begin{tabular}
{lccccccc|ccccc}
\toprule
& \multicolumn{7}{c|}{{Class-imbalanced} Graph Classification {(F1-ma.)}} & \multicolumn{5}{c}{{Class-imbalanced} Node Classification {(Acc.)}} \\
\cmidrule{2-13}
Variant &  {MUTAG} & {PROTEINS} & {D\&D} & {NCI1} &  {{PTC-MR}} & {{DHFR}} & {{REDDIT-B}} &{Cora-LT} &{CiteSeer-LT} &{PubMed-LT} &A. Photos & A. Computer\\
\midrule
GraphDec &\textbf{85.71} &\textbf{76.92} &\textbf{77.97} &\textbf{76.30} &\textbf{54.03} &64.25 &\textbf{69.70} &\textbf{78.29} &\textbf{66.90} &\textbf{78.20} &\textbf{93.85} &\textbf{92.19} \\
\midrule
\rowcolor{red!17}w/o GS  &80.10 &72.49 &63.16 &72.83 &48.48 &48.57 &61.40 &68.96 &60.33 &56.22 &73.22 &67.84 \\
\rowcolor{red!17}w/o SS  &80.95 &72.44 &72.26 &73.85 &52.96 &63.99 &70.61 &77.15 &64.67 &76.15 &79.09 &91.33 \\
\rowcolor{red!17}w/o CAD &78.41 &65.79 &68.33 &71.05 &52.13 &50.00 &67.15 &74.87 &62.62 &75.35 &90.71 &83.23 \\
\rowcolor{red!17}w/o RS &83.21 &67.21 &70.75 &71.08 &39.29 &60.99 &67.61 &73.27 &61.32 &72.02 &87.11 &90.38 \\
\midrule
\rowcolor{blue!17}w/o RM &44.37 &38.41 &65.30 &34.39 &32.14 &43.75 &64.82 &70.97 &54.58 &70.16 &79.01 &65.38 \\
\rowcolor{blue!17}w/o SG  &82.63 &74.54 &75.75 &70.13 &39.29 &62.44 &69.16 &77.54 &67.43 &72.43 &91.25 &90.05 \\
\rowcolor{blue!17}w/o CAG &83.50 &61.02 &69.23 &72.83 &41.18 &62.41 &64.14 &75.78 &63.43 &73.07 &92.77 &87.40 \\
\rowcolor{blue!17}w/o RW  &79.25 &65.54 &67.37 &72.99 &45.90 &61.53 &63.16 &76.46 &65.36 &75.54 &90.54 &89.10 \\
\midrule
w/o S.S.         &80.07 &71.90 &73.79 &69.72 &45.58 &\textbf{64.56} &65.67 &74.82 &65.28 &74.00 &86.14 &86.40 \\
\bottomrule
\end{tabular}}}
\vspace{-0.25in}
\end{table}

\vspace{-0.1in}
\subsection{Ablation Study}
\vspace{-0.05in}
Since GraphDec is a unified learning framework composed of multiple components (steps) and explores dynamic sparsity training from both model and dataset perspectives, we conduct ablation study to evaluate the performance of different model variants. 
{Specifically, GraphDec contains four components to address data sparsity and imbalance, including pruning samples by ranking gradients (GS), training with sparse dataset (SS), using cosine annealing to reduce dataset size (CAD), and recycling removed samples (RS),
and another four components to address model sparsity and data imbalance, including pruning weights by ranking magnitudes (RM), using sparse GNN (SG), using cosine annealing to progressively reduce sparse GNN's size (CAG), and reactivate removed weights (RW).
}
In addition, GraphDec employs self-supervision to calculate the gradient score. The details of model variants are provided in the Section B of Appendix.
We analyze the contributions of different components by removing each of them independently. 
We conduct experiments for both tasks to comprehensively inspect each component. The results are shown in Table \ref{tab:ablation-on-graph-sparsity}.

From the table, we find that the performance drops after removing any component, which demonstrates the effectiveness of each component in enhancing the model performance. In general, 
both mechanisms for addressing data and model sparsity
contribute significantly to the overall performance, demonstrating the necessity of these two mechanisms in solving sparsity problem.
Self-supervision is also essential, contributing similarly to dynamic sparsity mechanisms. Besides, it enables us to identify informative data samples without human labels and capture graph knowledge in a self-supervised manner.
In the dataset dynamic sparsity mechanism, GS and CAD contribute the most as sparse GNN's discriminability identifies hidden dynamic sparse subsets from the entire dataset accurately and efficiently.
{Regarding the model dynamic sparsity mechanism, removing RM and SG 
lead to a significant performance drop, which demonstrates that they are the key components in training the dynamic sparse GNN from the full GNN model. }
In particular, CAG enables the performance stability after the model pruning and helps capture information samples during decantation by assigning greater gradient norm.
Among these variants, the full model GraphDec achieves the best result in most cases. This demonstrates the effectiveness of dataset dynamic sparsity mechanism, model dynamic sparsity mechanism, and self-supervision strategy in our model.

\begin{figure*}[t]
    \centering
    {\includegraphics[width=1.0\textwidth]{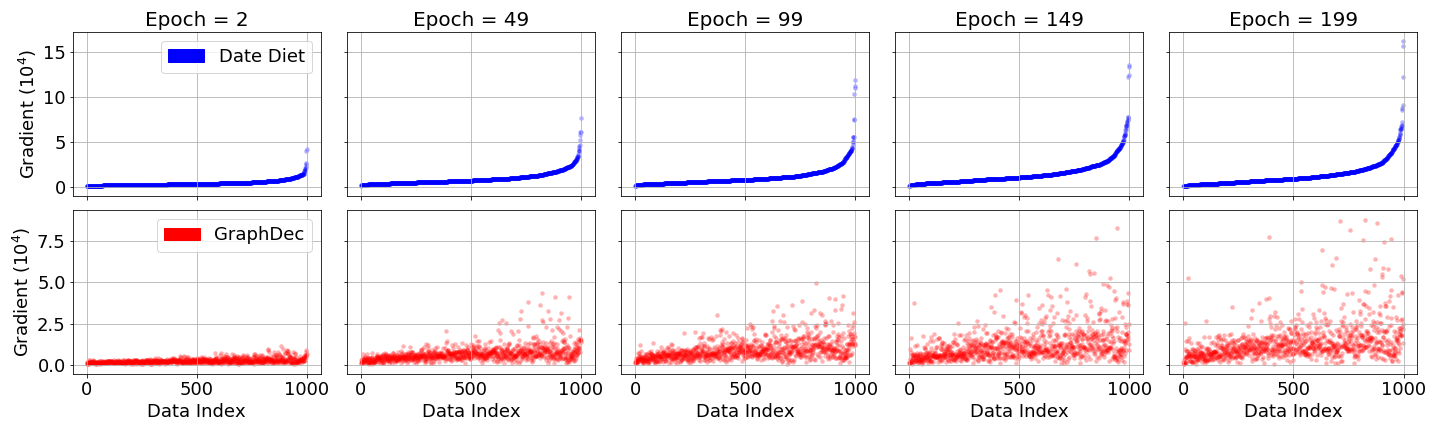}}
    \vspace{-0.25in}
    \caption{Evolution of data samples' gradients computed by {data diet~\cite{NEURIPS2021_ac56f8fe}} (upper figures) and our GraphDec (lower figures) on NCI1 data.}
     \vspace{-0.25in}
\label{fig:evolution}
\end{figure*}

\vspace{-0.15in}
\subsection{Analyzing Evolution of Sparse Subset by Scoring All Samples}
\vspace{-0.1in}
\label{subsec:evo-subset}
To show GraphDec's capability in dynamically identifying informative samples, we show the visualization of sparse subset evolution of data diet and GraphDec on {class-imbalanced} NCI1 dataset in Figure \ref{fig:evolution}.
Specifically, we compute 1000 graph samples with their importance scores. These samples are then ranked according to their scores and marked with sample indexes.
From the upper figures in Figure \ref{fig:evolution}, we find that data diet is unable to accurately identify the dynamic informative nodes. Once a data sample has been removed from the training list due to the low score, the model forever disregards it as unimportant. However, the fact that a sample is currently unimportant does not imply that it will remain unimportant indefinitely. Especially when the model cannot detect the true importance of each sample in the early stage, it may lead to the premature elimination of vital nodes.
Similarly, if a data sample is considered as important at early epochs (i.e., marked with higher sample index), it cannot be removed during subsequent epochs.
Therefore, we observe that data diet can only increase the scores of samples within the high index range (i.e., 500–1000), while ignoring samples within the low index range (i.e., <500). 
However, GraphDec (Figure \ref{fig:evolution} (bottom)) can capture the dynamic importance of each sample regardless of the initial importance score.
We see that samples with different indexes all have the opportunity to be considered important and therefore be included in the training list. Correspondingly, GraphDec takes into account a broader range of data samples when shrinking the training list, meanwhile maintaining flexibility towards the previous importance scores.

\vspace{-0.1in}
\section{Conclusion}
\vspace{-0.1in}
\label{sec:conclusion}
In this paper, to address the graph data imbalance challenge, we propose an efficient and effective method named \textbf{Graph} \textbf{Dec}antation (GraphDec). GraphDec leverages a dynamic sparse graph contrastive learning model to dynamically identified a sparse-but-informative subset for model training
, in which the sparse GNN encoder is dynamically sampled from a dense GNN, and its {capability of identifying informative samples} is used to rank and update the training data in each epoch. 
Extensive experiments demonstrate that GraphDec outperforms state-of-the-art baseline methods for both node classification and graph classification tasks in the {class-imbalanced} scenario. The analysis of the sparse informative samples' evolution further explains the superiority of GraphDec catching the informative subset in different training periods effectively.


\clearpage
\bibliographystyle{plain}
\clearpage
\appendix
\section{Proof of Theorem~\ref{app-thm}}
\label{app-proof}
\begin{theorem}
\label{app-thm}
For a data selection algorithm~\cite{killamsetty2021grad, NEURIPS2021_ac56f8fe}, we assume model training is optimized with full gradient descent. 
At $t \in \left[1, T\right]$ epoch, we denote the model's parameter as $\theta^{(t)}$ (satisfying ${\left\Vert\theta^{(t)}\right\Vert}^2 \leq {d}^2$, d is constant), the optimal model's parameter as $\theta^*$, subset data as $\mathcal{D}_{S}^{(t)}$, learning rate as $\alpha$. We also introduce the gradient error term as 
$\mbox{Err}(\mathcal{D}_{S}^{(t)}, \mathcal{L}, \mathcal{L}_{train}, \theta^{(t)}) = {\left\Vert \sum_{i \in \mathcal{D}_{S}^{(t)}} \nabla_{\theta}\mathcal{L}_{train}^i(\theta^{(t)}) -  \nabla_{\theta}\mathcal{L}(\theta^{(t)})\right\Vert}$, where 
$\mathcal{L}$ denotes training loss $\mathcal{L}_{train}$ over full training data or validation loss $\mathcal{L}_{val}$ over full validation data and $\mathcal{L}$ is a convex function. Then we have following guarantee:

If $\mathcal{L}_{train}$ is Lipschitz continuous with parameter $\sigma_T$ and $\alpha = \frac{d}{\sigma_T \sqrt{T}}$, then $\min_{t = 1:T} \mathcal{L}(\theta^{(t)}) - \mathcal{L}(\theta^*) \leq \frac{d\sigma_T}{\sqrt{T}} + \frac{d}{T}\sum_{t=1}^{T-1} \mbox{Err}(\mathcal{D}_{S}^{(t)}, \mathcal{L}, \mathcal{L}_{train}, \theta^{(t)})$. \\
\end{theorem}

\begin{proof}

The gradients of $\mathcal{L}_{val}$ and $\mathcal{L}_{train}$ are supposed to be $\sigma$-bounded by $\sigma_V$ and $\sigma_T$ respectively. 
According to gradient descent, we have:
\begin{equation}
\begin{aligned}
{\nabla_{\theta}\mathcal{L}_{train}(\theta^{(t)})}^\mathrm{T}(\theta^{(t)} - \theta^*) = \frac{1}{\alpha^{(t)}}{(\theta^{(t)} - \theta^{(t+1)})}^\mathrm{T}(\theta^{(t)} - \theta^*),
\end{aligned}
\end{equation}

\begin{equation}
\begin{aligned}
{\nabla_{\theta}\mathcal{L}_{train}(\theta^{(t)})}^\mathrm{T}(\theta^{(t)} - \theta^*) = \frac{1}{2\alpha^{(t)}}\left({\left\Vert\theta^{(t)} - \theta^{(t+1)}\right\Vert}^2 + {\left\Vert\theta_{t} - \theta^{*}\right\Vert}^2 - {\left\Vert\theta^{(t+1)} - \theta^{*}\right\Vert}^2\right).
\end{aligned}
\end{equation}

Since one update step $\theta^{(t)} - \theta^{(t+1)}$ can be optimized by gradient multiplying with learning rate $\alpha^{(t)} \nabla_{\theta}\mathcal{L}_{train}(\theta^{(t)})$, we have:

\begin{equation}
\begin{aligned}
\label{eq-24}
{\nabla_{\theta}\mathcal{L}_{train}(\theta^{(t)})}^\mathrm{T}(\theta^{(t)} - \theta^*) = \frac{1}{2\alpha^{(t)}}\left({\left\Vert \alpha^{(t)} \nabla_{\theta}\mathcal{L}_{train}(\theta^{(t)})\right\Vert}^2 + {\left\Vert\theta_{t} - \theta^{*}\right\Vert}^2 - {\left\Vert\theta^{(t+1)} - \theta^{*}\right\Vert}^2 \right).
\end{aligned}
\end{equation}

Since ${\nabla_{\theta}\mathcal{L}_{train}(\theta^{(t)})}^\mathrm{T}(\theta^{(t)} - \theta^*)$ can be represented as follows:
\begin{equation}
\begin{aligned}
\label{eq-25}
{\nabla_{\theta}\mathcal{L}_{train}(\theta^{(t)})}^\mathrm{T}(\theta^{(t)} - \theta^*) = {\nabla_{\theta}\mathcal{L}_{train}(\theta^{(t)})}^\mathrm{T}(\theta^{(t)} - \theta^*)\\ -{\nabla_{\theta}\mathcal{L}(\theta^{(t)})}^\mathrm{T}(\theta^{(t)} - \theta^*) + {\nabla_{\theta}\mathcal{L}(\theta^{(t)})}^\mathrm{T}(\theta^{(t)} - \theta^*),
\end{aligned}
\end{equation}

then based on the combination of the Equation~\eqref{eq-24} and Equation~\eqref{eq-25}, we have:
\begin{equation}
\begin{aligned}
{\nabla_{\theta}\mathcal{L}_{train}(\theta^{(t)})}^\mathrm{T}(\theta^{(t)} - \theta^*) -{\nabla_{\theta}\mathcal{L}(\theta^{(t)})}^\mathrm{T}(\theta^{(t)} - \theta^*) + {\nabla_{\theta}\mathcal{L}(\theta^{(t)})}^\mathrm{T}(\theta^{(t)} - \theta^*) =\\ \frac{1}{2\alpha^{(t)}}\left({\left\Vert \alpha^{(t)} \nabla_{\theta}\mathcal{L}_{train}(\theta^{(t)})\right\Vert}^2 + {\left\Vert\theta_{t} - \theta^{*}\right\Vert}^2 - {\left\Vert\theta^{(t+1)} - \theta^{*}\right\Vert}^2 \right)
\end{aligned}
\end{equation}

\begin{equation}
\begin{aligned}
\label{gd-cndtn}
{\nabla_{\theta}\mathcal{L}(\theta^{(t)})}^\mathrm{T}(\theta^{(t)} - \theta^*) = \frac{1}{2\alpha^{(t)}}\left({\left\Vert \alpha^{(t)} \nabla_{\theta}\mathcal{L}_{train}(\theta^{(t)})\right\Vert}^2 + {\left\Vert\theta_{t} - \theta^{*}\right\Vert}^2 - {\left\Vert\theta^{(t+1)} - \theta^{*}\right\Vert}^2 \right) \\ - {\left(\nabla_{\theta}\mathcal{L}_{train}(\theta^{(t)}) - \nabla_{\theta}\mathcal{L}(\theta^{(t)})\right)}^\mathrm{T}(\theta^{(t)} - \theta^*).
\end{aligned}
\end{equation}

We assume learning rate $\alpha^{(t)}, t \in [0, T-1]$ is a constant value, then we have:
\begin{equation}
\begin{aligned}
\sum_{t=0}^{T-1}{\nabla_{\theta}\mathcal{L}(\theta^{(t)})}^\mathrm{T}(\theta^{(t)} - \theta^*) = & \frac{1}{2\alpha}{\left\Vert\theta_{0} - \theta^{*}\right\Vert}^2 - {\left\Vert\theta_{T} - \theta^{*}\right\Vert}^2 + \sum_{t=0}^{T-1}(\frac{1}{2\alpha}{\left\Vert \alpha \nabla_{\theta}\mathcal{L}_{train}(\theta^{(t)})\right\Vert}^2) \nonumber\\ 
&+ \sum_{t=0}^{T-1}\left({\left(\nabla_{\theta}\mathcal{L}_{train}(\theta^{(t)}) -  \nabla_{\theta}\mathcal{L}(\theta^{(t)})\right)}^\mathrm{T}(\theta^{(t)} - \theta^*) \right).
\end{aligned}
\end{equation}

Since we assume ${\left\Vert\theta_{T} - \theta^{*}\right\Vert}^2 \geq 0$, then we have:
\begin{equation}
\begin{aligned}
\sum_{t=0}^{T-1}{\nabla_{\theta}\mathcal{L}(\theta^{(t)})}^\mathrm{T}(\theta^{(t)} - \theta^*) \leq \frac{1}{2\alpha}{\left\Vert\theta_{0} - \theta^{*}\right\Vert}^2 + \sum_{t=0}^{T-1}(\frac{1}{2\alpha}{\left\Vert \alpha \nabla_{\theta}\mathcal{L}_{train}(\theta^{(t)})\right\Vert}^2) \\+ \sum_{t=0}^{T-1}\left({\left(\nabla_{\theta}\mathcal{L}_{train}(\theta^{(t)}) -  \nabla_{\theta}\mathcal{L}(\theta^{(t)})\right)}^\mathrm{T}(\theta^{(t)} - \theta^*) \right).
\label{gd-equation}
\end{aligned}
\end{equation}

We assume $\mathcal{L}$ is convex and $\mathcal{L}_{train}$ is lipschitz continuous with parameter $\sigma_T$. Then for convex function $\mathcal{L}(\theta)$, we have $\mathcal{L}(\theta^{(t)}) - \mathcal{L}(\theta^*) \leq {\nabla_{\theta}\mathcal{L}(\theta^{(t)})}^\mathrm{T}(\theta^{(t)} - \theta^*)$. By combining this result with Equation~\ref{gd-equation}, we get:

\begin{equation}
\begin{aligned}
\sum_{t=0}^{T-1} \mathcal{L}(\theta^{(t)}) - \mathcal{L}(\theta^*) \leq  \frac{1}{2\alpha}{\left\Vert\theta_{0} - \theta^{*}\right\Vert}^2 + \sum_{t=0}^{T-1}(\frac{1}{2\alpha}{\left\Vert \alpha \nabla_{\theta}\mathcal{L}_{train}(\theta^{(t)})\right\Vert}^2) \\+ \sum_{t=0}^{T-1}\left({\left(\nabla_{\theta}\mathcal{L}_{train}(\theta^{(t)}) -  \nabla_{\theta}\mathcal{L}(\theta^{(t)})\right)}^\mathrm{T}(\theta^{(t)} - \theta^*) \right).
\end{aligned}
\end{equation}

Since $\left\Vert L_T(\theta)\right\Vert \leq \sigma_T$, ${\left\Vert \alpha \nabla_{\theta}\mathcal{L}_{train}(\theta^{(t)})\right\Vert} \leq \sigma_T$, and we assume $\left\Vert\theta - \theta^{*}\right\Vert \leq d$, then we have:

\begin{equation}
\begin{aligned}
\sum_{t=0}^{T-1} \mathcal{L}(\theta^{(t)}) - \mathcal{L}(\theta^*) \leq \frac{d^2}{2\alpha} + \frac{T \alpha \sigma_T^2}{2} + \sum_{t=0}^{T-1}d\left({\left\Vert\nabla_{\theta}\mathcal{L}_{train}(\theta^{(t)}) -  \nabla_{\theta}\mathcal{L}(\theta^{(t)})\right\Vert} \right),
\end{aligned}
\end{equation}

\begin{equation}
\begin{aligned}
\label{avg_loss}
\frac{1}{T}\sum_{t=0}^{T-1}\mathcal{L}(\theta^{(t)}) - \mathcal{L}(\theta^*) \leq \frac{d^2}{2\alpha T} + \frac{\alpha \sigma_T^2}{2} + \sum_{t=0}^{T-1}\frac{d}{T}\left({\left\Vert\nabla_{\theta}\mathcal{L}_{train}(\theta^{(t)}) - \nabla_{\theta}\mathcal{L}(\theta^{(t)})\right\Vert} \right).
\end{aligned}
\end{equation}

Since $\min{(\mathcal{L}(\theta^{(t)}) - \mathcal{L}(\theta^*))} \leq \frac{1}{T}\sum_{t=0}^{T-1}\mathcal{L}(\theta^{(t)}) - \mathcal{L}(\theta^*)$, based on Equation~\ref{avg_loss}, we have:
\begin{equation}
\begin{aligned}
\min{(\mathcal{L}(\theta^{(t)}) - \mathcal{L}(\theta^*))} \leq \frac{d^2}{2\alpha T} + \frac{\alpha \sigma_T^2}{2} + \sum_{t=0}^{T-1}\frac{d}{T}\left({\left\Vert\nabla_{\theta}\mathcal{L}_{train}(\theta^{(t)}) -  \nabla_{\theta}\mathcal{L}(\theta^{(t)})\right\Vert} \right).
\end{aligned}
\end{equation}

We set learning rate $\alpha = \frac{d}{\sigma_T \sqrt{T}}$ and then have:
\begin{equation}
\begin{aligned}
\label{min_loss}
\min{(\mathcal{L}(\theta^{(t)}) - \mathcal{L}(\theta^*))} \leq \frac{d\sigma_T}{\sqrt{T}} + \sum_{t=0}^{T-1}\frac{d}{T}\left({\left\Vert\nabla_{\theta}\mathcal{L}_{train}(\theta^{(t)})-  \nabla_{\theta}\mathcal{L}(\theta^{(t)})\right\Vert} \right).
\end{aligned}
\end{equation}

\end{proof}

\section{Experimental Details}
\subsection{Datasets Details}

In this work, seven graph classification datasets and five node classification datasets are used to evaluate the effectiveness of our proposed model, we provided their detailed statistics in Table \ref{tab:append-dataset}. 
For graph classification datasets, we follow the imbalance setting of \cite{wang2021imbalanced} to set the train-validation split as 25\%/25\% and change the imbalance ratio from 5:5 (balanced) to 1:9 (imbalanced). The rest of the dataset is used as the test set. The specified imbalance ratio of each dataset is clarified after its name in Table~\ref{tab:appendix-graph-cls-1}.
For node classification datasets, we follow~\cite{sen2008collective} to set the imbalance ratio of Cora, CiteSeer and PubMed as 10. Besides, the setting of Amazon-Photo and Amazon-Computers are borrowed from~\cite{park2022graphens}, where the imbalance ratio $\rho$ is set as 82 and 244, respectively.

\begin{table}[t]
    \centering
    \caption{Original dataset details for imbalanced graph classification and imbalanced node classification tasks.}
        \label{tab:append-dataset}
        \setlength{\tabcolsep}{1mm}{
        \scalebox{1.0}{
        \begin{tabular}{>{\small}c|>{\small}l|>{\small}c|>{\small}c|>{\small}c|>{\small}c|>{\small}c}
        \toprule
        Task & Dataset & \# Graphs & \# Nodes & \# Edges & \# Features & \# Classes \\
        \midrule
        \multirow{7}{*}{Graph} 
        & {MUTAG} & 188 & $\sim$17.93 & $\sim$19.79 & - & 2 \\
        & {PROTEINS} & 1,113 & $\sim$39.06 & $\sim$72.82 & - & 2\\
        & {D\&D} & 1,178 & $\sim$284.32 & $\sim$715.66 & - & 2\\
        & {NCI1} & 4,110 & $\sim$29.87 & $\sim$32.30 & - & 2\\
        & {PTC-MR} & 344 & $\sim$14.29 & $\sim$14.69 & - & 2\\
        & {DHFR} & 756 & $\sim$42.43 & $\sim$44.54 & - & 2\\
        & {REDDIT-B} & 2,000 & $\sim$429.63 & $\sim$497.75 & - & 2\\
        \midrule
        \multirow{5}{*}{Node}
        & {Cora} & - & 2,485 & 5,069 & 1,433 & 7\\
        & {Citeseer} & - & 2,110 & 3,668 & 3,703 & 6\\
        & {Pubmed} & - & 19,717 & 44,324 & 500 & 3\\
        & {A-photo} & - & 7,650 & 238,162 & 745 & 8\\
        & {A-computers} & - & 13,381 & 245,778 & 767 & 10\\
        \bottomrule
    \end{tabular}
    }
    }
        \vspace{-0.1in}
\end{table}

\subsection{Baseline Details}
We compare our model with a variety of baseline methods using different rebalance methods:

I. For \textbf{imbalanced graph classification}~\cite{wang2021imbalanced}, four models are included as baselines in our work, we list these baselines as follow:

(1) \textbf{GIN}~\cite{xu2018how}, a popular supervised GNN backbone for graph tasks due to its powerful expressiveness on graph structure;

(2) \textbf{InfoGraph}~\cite{sun2019infograph}, an unsupervised graph learning framework by maximizing the mutual information between the whole graph and its local topology of different levels;

(3) \textbf{GraphCL}~\cite{You2020GraphCL}, learning unsupervised graph representations via maximizing the mutual information between the original graph and corresponding augmented views;

(4) \textbf{G$^2$GNN}~\cite{wang2021imbalanced}, a re-balanced GNN proposed to utilize additional supervisory signals from both neighboring graphs and graphs themselves to alleviate the imbalance issue of graph.

II. For \textbf{imbalanced node classification}, we consider nine baseline methods in our work, including 

(1) \textbf{vanilla}, denoting that we train GCN normally without any extra rebalancing tricks;

(2) \textbf{re-weight}~\cite{japkowicz2002class}, denoting we use cost-sensitive loss and re-weight the penalty of nodes in different classes;

(3) \textbf{oversampling}~\cite{park2022graphens}, denoting that we sample nodes of each class to make the data's number of each class reach the maximum number of corresponding class's data;

(4) \textbf{cRT}~\cite{Kang2020Decoupling}, a post-hoc correction method for decoupling output representations;

(5) \textbf{PC Softmax}~\cite{hong2021disentangling}, a post-hoc correction method for decoupling output representations, too;

(6) \textbf{DR-GCN}~\cite{ijcai2020-398}, building virtual minority nodes and forces their features to be close to the neighbors of a source minority node;

(7) \textbf{GraphSMOTE}~\cite{zhao2021graphsmote}, a pre-processing method that focuses on the input data and investigates the possibility of re-creating new nodes with minority features to balance the training data.

(8) \textbf{GraphENS}~\cite{park2022graphens}, proposing a new augmentation method to construct an ego network from all nodes for learning minority representation. 

We use Graph Convolutional Network (GCN) \cite{kipf2017semi} as the default architecture for all rebalance methods.
\subsection{Details of GraphDec Variants}
The details of model variants are provided as follows:

I. Specifically, GraphDec contains four components to address data sparsity and imbalance: 
(1) \textbf{GS} is sampling informative subset data according to ranking gradients; 
(2) \textbf{SS} is training model with the sparse dataset, correspondingly; 
(3) \textbf{CAD} is using cosine annealing to reduce dataset size; 
(4) \textbf{RS} is recycling removed samples, correspondingly. 
To investigate their corresponding effectiveness, we remove them correspondingly as:

(1) \textbf{w/o} GS is that we randomly sample subset from the full set;

(2) \textbf{w/o} SS is that we train GNN with the full set;

(3) \textbf{w/o} CAD is that we directly reduce dataset size to target dataset size and it is same as data diet;

(4) \textbf{w/o} RS is not recycling any removed samples.

II. Another four components to address model sparsity and data imbalance:
(1) \textbf{RM} samples model weights according to ranking magnitudes;
(2) \textbf{SG} is using sparse GNN, correspondingly;
(3) \textbf{CAG} is using cosine annealing to progressively reduce sparse GNN's size;
(4) \textbf{RW} is reactivating removed weights.
To investigate their effectiveness, we remove them correspondingly as:

(1) \textbf{w/o RM} is that we randomly sample activated weights from full GNN model;

(2) \textbf{w/o SG} is that we train full GNN during forward and backward;

(3) \textbf{w/o CAG} is that we directly reduce the model size to target sparsity rate;

(4) \textbf{w/o RW} is not reactivating any removed weights during sparse training.

\subsection{Full Results with Error Bars}
\begin{table*}[t]
\caption{Imbalanced graph classification results.
The numbers after each dataset name indicate the imbalance ratios of minority to majority categories. We report the macro F1-score and micro F1-score with the standard errors as Results are reported as $mean\pm std$ for 3 repetitions on each dataset. We bold the best performance.
}
\renewcommand\arraystretch{1}
\centering
\resizebox{1\textwidth}{!}{
\begin{tabular}{l|c|cc|cc|cc|cc}
\toprule
{Rebalance} &\multirow{2}{*}{Basis} & \multicolumn{2}{c|}{MUTAG (5:45)} &\multicolumn{2}{c|}{PROTEINS (30:270)} &\multicolumn{2}{c|}{D\&D (30:270)} &\multicolumn{2}{c}{NCI1 (100:900)} \\
\cmidrule{3-4} \cmidrule{5-6} \cmidrule{7-8}\cmidrule{9-10}  Method& & F1-ma. & F1-mi. & F1-ma. & F1-mi. & F1-ma. & F1-mi. & F1-ma. & F1-mi.  \\
\midrule
\multirow{3}{*}{vanilla} &GIN~\cite{xu2018how} &$52.50\pm 18.70$ & $56.77\pm 14.14$ & $25.33\pm 7.53$ & $ 28.50\pm 5.82$ & $9.99\pm 7.44$ & $11.88\pm 9.49$ & $18.24\pm 7.58$ & $18.94 \pm 7.12$ \\
&InfoGraph~\cite{sun2019infograph} &$69.11\pm 9.03$ & $69.68\pm 7.77$ & $35.91\pm 7.58$ & $36.81\pm 6.51$ & $21.41\pm 4.51$ & $27.68\pm 7.52$ & $33.09\pm 3.30$ & $34.03\pm 3.68$\\
&GraphCL~\cite{You2020GraphCL} & $66.82\pm 11.56$ & $67.77\pm 9.78$ & $40.86\pm 6.94$  &  $41.24\pm 6.38$& $21.02\pm 3.05$ & $26.80\pm 4.95$ & $31.02\pm 2.69$ & $31.62\pm 3.05$ \\
\midrule
\multirow{3}{*}{up-sampling} &GIN~\cite{xu2018how} & $78.03\pm 7.62$ & $78.77\pm 7.67$ & $65.64\pm 2.67$ & $71.55\pm 3.19$ & $41.15\pm 3.74$ & $70.56\pm 10.28$ & $59.19\pm 4.39$ & $71.80\pm 7.02$ \\
&InfoGraph~\cite{sun2019infograph} & $78.62\pm 6.84$ & $79.09\pm 6.86$ & $62.68\pm 2.70$ & $66.02\pm 3.18$ & $41.55\pm 2.32$ & $71.34\pm 6.76$ & $53.38\pm 1.88$ & $62.20\pm 2.63$ \\
&GraphCL~\cite{You2020GraphCL} & $80.06\pm 7.79$ & $80.45\pm 7.86$ &$64.21\pm 2.53$  & $65.76\pm 2.61$ & $38.96\pm 3.01$ & $64.23\pm 8.10$ & $49.92\pm 2.15$ & $58.29\pm 3.30$ \\
\midrule
\multirow{3}{*}{re-weight} &GIN~\cite{xu2018how} & $77.00\pm 9.59$ & $77.68\pm 9.30$ & $54.54\pm 6.29$ & $55.77\pm 7.11$ & $28.49\pm 5.92$ & $40.79\pm 11.84$ & $36.84\pm 8.46$ & $39.19\pm 10.05$ \\
&InfoGraph~\cite{sun2019infograph} & $80.85\pm 7.75$ & $81.68\pm 7.83$ & $65.73\pm 3.10$& $69.60\pm 3.68$ & $41.92\pm 2.28$ & $72.43\pm 6.63$ & $53.05\pm 1.12$ & $62.45\pm 1.89$ \\
&GraphCL~\cite{You2020GraphCL} & $80.20\pm 7.27$ & $80.84\pm 7.43$ & $63.46\pm 2.42$ & $64.97\pm 2.41$ & $40.29\pm 3.31$ & $67.96\pm 8.98$ & $50.05\pm 2.09$ & $58.18\pm 3.08$ \\
\midrule
\multirow{2}{*}{G$^2$GNN~\cite{wang2021imbalanced}} &remove edge & $80.37\pm 6.73$ & $81.25\pm 6.87$ & \underline{$67.70\pm 2.96$} & $73.10\pm 4.05$ & $43.25\pm 3.91$ & \underline{$77.03\pm 9.98$} & $63.60\pm 1.57$ & $72.97\pm1.81$ \\
&mask node & \underline{$83.01\pm 7.01$} & \underline{$83.59\pm 7.14$} & $67.39\pm 2.99$ & \underline{$73.30\pm 4.19$} & \underline{$43.93\pm 3.46$} & $79.03\pm 10.78$ & \underline{$64.78\pm 2.86$} & \underline{$74.91\pm 2.14$} \\
\midrule
GraphDec &dynamic sparsity &\textbf{85.71$\pm$10.20} &\textbf{85.71$\pm$11.10} &\textbf{76.92$\pm$6.15} &\textbf{76.89$\pm$6.80} &\textbf{77.97$\pm$6.75} &\underline{77.02$\pm$6.26} &\textbf{76.30$\pm$5.12} &\textbf{76.29$\pm$6.27} \\
\bottomrule
\end{tabular}
}

\resizebox{1\textwidth}{!}{
\begin{tabular} {l|c|cc|cc|cc}
\toprule
{Rebalance} &\multirow{2}{*}{Basis} & \multicolumn{2}{c|}{{PTC-MR} (9:81)} &\multicolumn{2}{c|}{{DHFR} (12:108)} &\multicolumn{2}{c}{{REDDIT-B} (50:450)} \\ 
\cmidrule{3-4} \cmidrule{5-6} \cmidrule{7-8} Method& & F1-ma. & F1-mi. & F1-ma. & F1-mi. & F1-ma. & F1-mi. \\
\midrule
\multirow{3}{*}{vanilla} &GIN~\cite{xu2018how} &$17.74\pm 6.49$ & $20.30\pm 6.06$ & $35.96\pm 8.87$ & $49.46\pm 4.90$ & $33.19\pm 14.26$ & $36.02\pm 17.38$ \\
&InfoGraph~\cite{sun2019infograph} & $25.85\pm 6.14$ & $26.71\pm 6.50$ & $50.62\pm 8.33$ & $56.28\pm 4.58$ & $57.67\pm 3.80$ & $67.10\pm 4.91$ \\
&GraphCL~\cite{You2020GraphCL} &$24.22\pm 6.21$ & $25.16\pm 5.25$ & $50.55\pm 10.01$ & $56.31\pm 6.12$ & $53.40\pm 4.06$ & $62.19\pm 5.68$  \\
\midrule
\multirow{3}{*}{up-sampling} &GIN~\cite{xu2018how} & $44.78\pm 8.01$ & $55.43\pm 14.25$ & $55.96\pm 10.06$ & $59.39\pm 6.52$ & $66.71\pm 3.92$ & $83.00\pm 5.18$\\
&InfoGraph~\cite{sun2019infograph} & $44.29\pm 4.69$ & $48.91\pm 7.49$ & $59.49\pm 5.20$ & $61.62\pm 4.18$ & $67.01\pm 3.34$ & $78.68\pm 3.71$ \\
&GraphCL~\cite{You2020GraphCL} & $45.12\pm 7.33$ & $53.50\pm 13.31$ & $60.29\pm 9.04$ & $61.71\pm 6.75$ & $62.01\pm 3.97$ & $75.84\pm 3.98$\\
\midrule
\multirow{3}{*}{re-weight} &GIN~\cite{xu2018how} & $36.96\pm 14.08$ & $43.09\pm 20.01$ & $55.16\pm 9.47$ & $57.78\pm 6.69$ & $45.17\pm 8.46$ & $51.92\pm 12.29$  \\
&InfoGraph~\cite{sun2019infograph} & $44.09\pm 5.62$ & $49.17\pm 8.78$  & $58.67\pm 5.82$ & $60.24\pm 4.80$ & $65.79\pm 3.38$ & $77.35\pm 3.96$ \\
&GraphCL~\cite{You2020GraphCL} & $44.75\pm 7.62$ & $52.22\pm 13.24$ & $60.87\pm 6.33$ & $61.93\pm 5.15$ & $62.79\pm 6.93$ & $76.15\pm 9.15$\\
\midrule
\multirow{2}{*}{G$^2$GNN~\cite{wang2021imbalanced}} &remove edge & \underline{$46.40\pm 7.73$}  & \underline{$56.61\pm13.72$} & \underline{$61.63\pm 10.02$} & \underline{$63.61\pm 6.05$} & \underline{$68.39\pm 2.97$} & \underline{$86.35\pm 2.27$} \\
&mask node & $46.61\pm 8.27$ & $56.70\pm 14.81$ & $59.72\pm 6.83$ & $61.27\pm 5.40$ & $67.52\pm 2.60$ & $85.43\pm 1.80$ \\
\midrule
GraphDec &dynamic sparsity &\textbf{54.03$\pm$8.22} &\textbf{61.17$\pm$10.24} &\textbf{64.25$\pm$9.54} &\textbf{67.91$\pm$7.10} &\textbf{69.70$\pm$7.20} &\textbf{87.00$\pm$9.36} \\
\bottomrule
\end{tabular}}
\label{tab:appendix-graph-cls-1}
\end{table*}

\begin{table*}[t]
\caption{Imbalanced node classification results. We report the accuracy, balanced accuracy and macro F1-score with the standard errors as $mean\pm std$ for 3 repetitions on each dataset. We bold the best performance.
}
\renewcommand\arraystretch{1}
\centering
\resizebox{1\textwidth}{!}{
\begin{tabular} {l|ccc|ccc|ccc|cc|cc}
\toprule
\multirow{2}{*}{Method}  &\multicolumn{3}{c|}{Cora-LT} &\multicolumn{3}{c|}{CiteSeer-LT} & \multicolumn{3}{c|}{PubMed-LT} & \multicolumn{2}{c|}{A.P. ($\rho=$82)} &\multicolumn{2}{c}{A.C. ($\rho=$244)}  \\ 
\cmidrule{2-14} & Acc. & bAcc. & F1-ma. & Acc. & bAcc. & F1-ma. & Acc. & bAcc. & F1-ma. & (b)Acc. & F1-ma. & (b)Acc. & F1-ma. \\
\midrule
vanilla      &73.66$\pm$0.28 &62.72$\pm$0.39 &63.70$\pm$0.43 &53.90$\pm$0.70 &47.32$\pm$0.61 &43.00$\pm$0.70 &70.76$\pm$0.74 &57.56$\pm$0.59 &51.88$\pm$0.53 &82.86$\pm$0.30  &78.72$\pm$0.52  &68.47$\pm$2.19  &64.01$\pm$3.18  \\
Re-Weight~\cite{park2022graphens}    &75.20$\pm$0.19 &68.79$\pm$0.18 &69.27$\pm$0.26 &62.56$\pm$0.32 &55.80$\pm$0.28 &53.74$\pm$0.28 &77.44$\pm$0.21 &72.80$\pm$0.38 &73.66$\pm$0.27  &92.94$\pm$0.13  &92.95$\pm$0.13  &90.04$\pm$0.29  &90.11$\pm$0.28  \\
Oversampling~\cite{park2022graphens} &77.44$\pm$0.09 &70.73$\pm$0.10 &72.40$\pm$0.11 &62.78$\pm$0.37 &56.01$\pm$0.35 &53.99$\pm$0.37 &76.70$\pm$0.48 &68.49$\pm$0.28 &69.50$\pm$0.38  &92.46$\pm$0.47   &92.47$\pm$0.48   &89.79$\pm$0.16   &89.85$\pm$0.17  \\
cRT~\cite{Kang2020Decoupling}          &76.54$\pm$0.22 &69.26$\pm$0.48 &70.95$\pm$0.50 &60.60$\pm$0.25 &54.05$\pm$0.22 &52.36$\pm$0.22 &75.10$\pm$0.23 &67.52$\pm$0.72 &68.08$\pm$0.85  &91.24$\pm$0.28   &91.17$\pm$0.29   &86.02$\pm$0.55   &86.00$\pm$0.56   \\
PC Softmax~\cite{hong2021disentangling}   &76.42$\pm$0.34 &71.30$\pm$0.45 &71.24$\pm$0.52 &65.70$\pm$0.42 &61.54$\pm$0.45 &61.49$\pm$0.49 &76.92$\pm$0.26 &75.82$\pm$0.25 &74.19$\pm$0.25  &93.32$\pm$0.25   &93.32$\pm$0.25   &86.59$\pm$0.92   &86.62$\pm$0.91   \\
DR-GCN~\cite{ijcai2020-398}       &73.90$\pm$0.29 &64.30$\pm$0.39 &63.10$\pm$0.57 &56.18$\pm$1.10 &49.57$\pm$1.08 &44.98$\pm$1.29 &72.38$\pm$0.19 &58.86$\pm$0.15 &53.05$\pm$0.13   & N/A & N/A & N/A & N/A \\
GraphSmote~\cite{zhao2021graphsmote}   &76.76$\pm$0.31 &69.31$\pm$0.37 &70.21$\pm$0.64 &62.58$\pm$0.30 &55.94$\pm$0.34 &54.09$\pm$0.37 &75.98$\pm$0.22 &70.96$\pm$0.36 &71.85$\pm$0.32  &92.65$\pm$0.31  &92.61$\pm$0.32  &89.31$\pm$0.34   &89.39$\pm$0.35   \\
GraphENS~\cite{park2022graphens}     &\underline{77.76$\pm$0.09} &\underline{72.94$\pm$0.15} &\underline{73.13$\pm$0.11} &66.92$\pm$0.21 &60.19$\pm$0.21 &58.67$\pm$0.25 &\underline{78.12$\pm$0.06} &74.13$\pm$0.22 &\underline{74.58$\pm$0.13}  &\underline{93.82$\pm$0.13 }  &\underline{93.81$\pm$0.12}  &\underline{91.94$\pm$0.17}  &\underline{91.94$\pm$0.17}  \\
\midrule
GraphDec &\textbf{78.29$\pm$0.40} &\textbf{73.94$\pm$0.67} &\textbf{74.25$\pm$0.83} &\textbf{66.90$\pm$0.65} &\textbf{61.56$\pm$0.72} &\textbf{61.85$\pm$0.96} &\textbf{78.20$\pm$0.45} &\textbf{76.05$\pm$0.66} &\textbf{76.32$\pm$0.66}  &\textbf{93.85$\pm$0.72 }  &\textbf{94.02$\pm$0.67 }  &\textbf{92.19$\pm$0.73 }  &\textbf{92.16$\pm$0.75 } \\
\bottomrule
\end{tabular}}
\label{tab:appen-node-cls-1}
\end{table*}
We provide the F1-macro and F1-micro scores along with their standard deviation for our model and other baselines across both graph classification and node classification tasks in Table~\ref{tab:appendix-graph-cls-1} and Table~\ref{tab:appen-node-cls-1}. We report their results as $mean\pm std$ for 3 repetitions on each metric for each dataset.

\section{Finding Informative Samples by Sparse GNN}
\begin{figure}[t]
            \centering
    {\includegraphics[scale=0.4]{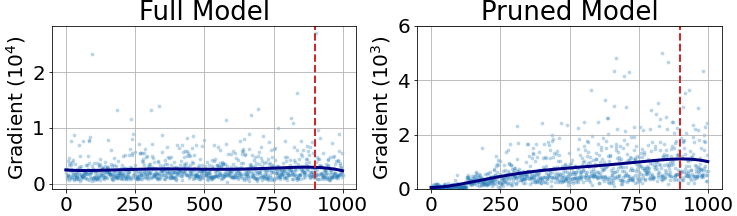}
}
    \caption{Results of data samples’ gradients computed by full GNN model and our dynamic sparse GNN model on NCI1 data. Red dashed line: on the left side, points on the x-axis [0, 900] are majority class; on the right side, points on the x-axis [900, 1000] are minority class.}
\label{fig:1x5}
\end{figure}
Compared with the full GNN model, our dynamic sparse GNN model is more sensitive to recognizing informative data samples which can be empirically verified by Figure~\ref{fig:1x5}. 
As we can see in the figure, our dynamical pruned model can assign larger gradients for minority-class samples than majority-class samples during contrastive training, while the full model generally assigns relatively uniform gradients for both minority-class and majority-class samples. Thus, the proposed dynamically pruned model demonstrates its discriminatory ability on minority-class and can thereby sample more minority-class data according to computed unsupervised gradients.

{
\section{Computational Cost}
To evaluate the proposed GraphDec's computational cost on a wide range of datasets, results in Table~\ref{tab:append-compute-cost} that include three different class-imbalanced node classification datasets (PubMed-LT, Cora-LT, CiteSeer-LT), three different class-imbalanced graph classification datasets (MUTAG, PROTEINS, PTC\_MR), and four baselines (vanilla GCN, re-weight, re(/over)-sample, GraphCL). We run 200 epochs for each method to measure their computational time (second) for training.
On NVIDIA GeForce RTX 3090 GPU device, we get the running time, as reported in Table~\ref{tab:append-compute-cost}. All models are implemented in PyTorch Geometric~\cite{Fey/Lenssen/2019}.
\begin{table}[t]
    \centering
    \caption{Computational time (second) comparisons.}
        \label{tab:append-compute-cost}
        \setlength{\tabcolsep}{1mm}{
        \scalebox{1.0}{
        \begin{tabular}{>{\small}c|>{\small}l|>{\small}c|>{\small}l|>{\small}c|>{\small}l|>{\small}c|>{\small}l}
        \toprule
        Model & Method & PubMed-LT & Cora-LT  & CiteSeer-LT  & PROTEINS   & PTC\_MR    & MUTAG   \\
        \midrule
        \multirow{5}{*}{GCN} 
        & {vanilla}       &   2.436   &    2.154     &     2.129 &      12.798 &   4.295 &     2.989\\
        & re-weight &          2.330 &       2.282 &           2.150 &        12.903 &       4.410 &    3.125\\
        & re(/over)-sample &          3.241 &       2.860 &           2.794 &        15.996 &       5.734 &    4.022\\
        & {GraphCL} & 3.747 &3.412           &3.399         &14.981        &5.049    &3.215\\
        & {GraphDec} & 2.243 &1.995          &1.952        &10.614       &4.212      &2.090\\
        \bottomrule
    \end{tabular}
    }
    }
\end{table}
According to the results, our GraphDec has less computation cost than prior methods. 
The following explains why augmentation doubles the input graph without increasing overall computation costs: (i) The augmentations we use (e.g, node dropping and edge dropping) reduces the size of input graphs (i.e., node number decreases 25\%, edge number decreases 25-35\%); (ii) During each epoch, our GraphDec prunes datasets so that only approximately 50\% of the training data is used. (iii) our GraphDec prunes GNN model weight, resulting in a lighter model during training. (iv) Despite the fact that augmentation doubles the number of input graphs, the additional new views only consume forward computational resources without requiring a backward step or weight update step, thereby only marginally increasing computation.
}
\end{document}